\documentclass[final,12pt]{colt2021} 

\title[
]{
Is Reinforcement Learning More Difficult Than Bandits?\\
A Near-optimal Algorithm Escaping the Curse of Horizon
	}
\usepackage{times}
\usepackage{xcolor}

\usepackage{hyperref}
\usepackage{url}

\usepackage{mathtools}
\usepackage{graphicx}
\usepackage{amssymb}
\usepackage{bbm}
\usepackage{xpatch}
\usepackage{mathrsfs}

\usepackage{url}
\usepackage{array}
\usepackage{wrapfig}
\usepackage[normalem]{ulem} 
\usepackage{enumerate}
\usepackage{enumitem}

\usepackage{footnote}

\newtheorem{assumption}{Assumption}

\def\qkhsa{$Q_h^k(s,a)$}

\def\vkh{$V_h^k$}

\def\vkh1{$V_{h+1}^k$}
\def\vkh1'{$V_{h+1}^k(s')$}

\def\ovh{$V_{h}^*(\cdot)$}

\def\ovh1{$V_{h+1}^*$}
\def\ovh1'{$V_{h+1}^*(s')$}

\def\tvkh{$ \tilde{V}_h^k $}

\def\tvkh1{ $\tilde{V}_{h+1}^k$ }
\def\tvkh1'{$\tilde{V}_{h+1}^k(s')$}

\def\pkhs'{$P_{s_h^k,a_h^k,s'} $}

\def\hpksa{$\hat{P}^k_{s,a}$}
\def\hpksa'{$\hat{P}^k_{s,a,s'}$}
\def\hpkh{$\hat{P}^k_{s_h^k,a_h^k}$}
\def\hpkh'{$ \hat{P}^k_{s_h^k,a_h^k,s'} $}

\newcommand{\abs}[1]{\left|#1\right|}
\newcommand{\expect}{\mathbb{E}}

\newcommand{\indict}{\mathbb{I}}
\newcommand{\states}{\mathcal{S}}
\newcommand{\trans}{P}
\newcommand{\actions}{\mathcal{A}}
\newcommand{\mdp}{M}

\usepackage{algorithm}
\usepackage{algorithmic}
\newcommand{\algo}{ {\sc MVP}  }

\newcommand{\poly}{\mathrm{poly}}
\newcommand{\cb}{\mathrm{CB}}

\usepackage{pifont}

\usepackage{hhline}
\usepackage{makecell}

\allowdisplaybreaks

\coltauthor{%
 \Name{Zihan Zhang} \Email{zihan-zh17@mails.tsinghua.edu.cn}\\
 \addr Tsinghua University
 \AND
 \Name{Xiangyang Ji} \Email{xyji@tsinghua.edu.cn}\\
 \addr Tsinghua University
 \AND
 \Name{Simon S. Du} \Email{ssdu@cs.washington.edu}\\
 \addr University of Washington
}

\begin{document}

\maketitle

\begin{abstract}\footnote{Accepted for presentation at the Conference on Learning Theory (COLT) 2021}Episodic reinforcement learning and contextual bandits are two widely studied sequential decision-making problems.
Episodic reinforcement learning generalizes contextual bandits and is often perceived to be more difficult due to long planning horizon and unknown state-dependent transitions.
The current paper shows that the long planning horizon and the unknown state-dependent transitions (at most) pose little additional difficulty on sample complexity.

We consider the episodic reinforcement learning with $S$ states, $A$ actions,  planning horizon $H$, total reward bounded by $1$, and the agent plays for $K$ episodes.
We propose a new algorithm, \textbf{M}onotonic \textbf{V}alue \textbf{P}ropagation (MVP), which relies on a new Bernstein-type bonus. Compared to existing bonus constructions, the new bonus is tighter since it is based on a well-designed monotonic value function. In particular, the \emph{constants} in the bonus should be subtly setting
to ensure optimism and monotonicity.

We show MVP enjoys an $O\left(\left(\sqrt{SAK} + S^2A\right) \poly\log \left(SAHK\right)\right)$ regret,
  approaching the $\Omega\left(\sqrt{SAK}\right)$ lower bound of \emph{contextual bandits} up to logarithmic terms.
Notably, this result 1) \emph{exponentially} improves the state-of-the-art polynomial-time algorithms by Dann et al. [2019] and Zanette et al. [2019] in terms of the dependency on $H$, and 2) \emph{exponentially} improves the running time in [Wang et al. 2020] and significantly improves the dependency on $S$, $A$ and $K$ in sample complexity.
\end{abstract}

\section{Introduction}
\label{sec:intro}
Episodic reinforcement learning (RL) and contextual bandits (CB) are two representative sequential decision-making problems.
RL is a strict generalization of CB and is often perceived to be much more difficult due to the additional two challenges that are absent in CB: 1) long planning horizon and 2) unknown state-dependent transitions.
These two challenges in RL requires the agent to  not only consider the immediate reward but also the possible transitions into differing states in the long run.
On the other hand, one can view CB as a episodic RL problem with a horizon equal to one.\footnote{See Section~\ref{sec:pre} for the precise correspondence.}
In CB, it is sufficient to act myopically by choosing the action which maximizes the immediate reward.

Although RL and CB are widely studied in the literature, somehow surprisingly, the following fundamental problem remains open:
\begin{center}
\textbf{
	Does episodic reinforcement learning require more samples than contextual bandits?
}
\end{center}
Here the sample complexity is measured in terms of regret or the number of episodes to learn a near-optimal policy.
To put it differently, this question asks whether the long planning horizon and/or the unknown state-dependent transitions pose additional difficulty.

\citet{jiang2018open} conjectured that  for
tabular, episodic RL problems, under the assumption that the total reward is bounded by $1$, \footnote{
	This assumption is made in order to have a fair comparison with CB.
See Section~\ref{sec:rel} for discussions.}
there exists an $\Omega\left(\frac{SAH}{\epsilon^2}\right)$ PAC learning, or analogically, an $\Omega\left(\sqrt{SAHK}\right)$ regret  lower bound, where $S$ is the number of states, $A$ is the number of actions, $H$ is the planning horizon, $\epsilon$ is the target sub-optimality and $K$ is the total number of episodes.
In contrast, it is well know that for CB, one can achieve an $\widetilde{O}\left(\frac{SA}{\epsilon^2}\right)$ PAC learning or an $\widetilde{O}\left(\sqrt{SAK}\right)$ regret upper bound.\footnote{Throughout the paper, $\widetilde{O}\left(\cdot\right)$ omits logarithmic factors.}
If this conjecture is true, then there is a formal  sample complexity separation between RL and CB.

However, this conjecture was recently refuted by \citet{wang2020long}, who presented a new method which enjoys an $O\left(\frac{S^5A^4 \poly\log \left(HSA /\epsilon\right) }{\epsilon^3}\right)$ PAC learning upper bound, the first bound that has only a \emph{logarithmic dependency on $H$}.
This encouraging result gives the hope: \emph{episodic reinforcement learning is as easy as contextual bandit} in terms of the sample complexity.
Furthermore, this claim would convey a  conceptual message in a sense that long planning horizon and unknown state-dependent transitions pose no additional difficulty in sequential decision-making problems.

To formally establish this claim, we need to design an algorithm which enjoys an $O\left(\frac{SA}{\epsilon^2}\right)$ PAC learning and an $O\left(\sqrt{SAK}\right)$ regret upper bounds, which match the sample complexity lower bounds of CB.
Ideally, we would also like this algorithm to be computationally efficient.
The result in \cite{wang2020long} is still far from this grand goal, as its  dependencies on $S$, $A$ and $\epsilon$ are suboptimal and their algorithm runs in exponential time.
See Section~\ref{sec:rel} for more discussions.
Indeed, \citet{wang2020long} listed two open problems: 1) to develop an algorithm with  sample complexity $\widetilde{O}\left(\frac{SA}{\epsilon^2}\right)$ or regret $\widetilde{O}\left(\sqrt{SAK}\right)$ and 2) to develop a polynomial-time algorithm whose sample complexity scales logarithmically with $H$.

\subsection{Main Results}
In this paper, we take an important step toward  this grand goal.
We design an upper confidence bound (UCB)-based algorithm, \textbf{M}onotonic \textbf{V}alue \textbf{P}ropogation (MVP), which enjoys the following sample complexity bounds.

\begin{theorem}\label{thm1} 
Suppose the reward is non-negative and the total reward at every episode is bounded by $1$.
For any $K \geq 1$ and $\delta \in (0,1)$, we have that with probability $1-\delta$, the regret of \algo is bounded by $\mathrm{Regret}(K) = O\left(\left(\sqrt{SAK}+S^2A\right) \poly\log\left(SAHK/\delta\right)\right)$.
%
\end{theorem}
Using a standard reduction (see Section~\ref{sec:pre}), we can show that we can find an $\epsilon$-suboptimal policy in $O\left(\left(\frac{SA}{\epsilon^2} + \frac{S^2 A }{\epsilon}\right)\poly\log\left(\frac{SAH}{\epsilon\delta}\right)\right)$ episodes.

Our results are significant in the following senses.
\begin{enumerate*}
	\item These bounds match the information theoretical lower bound of CB up to logarithmic factors in the regime where the number of episodes is moderately large, $K = \widetilde{\Omega}\left(S^3A\right)$ or the target accuracy is moderately small, $\epsilon = \widetilde{O}\left(1/S\right)$.
	Our result thus significantly closes the gap between RL and CB.
	\item 
	\algo is the first computationally efficient algorithm whose sample complexity scales logarithmically with $H$, and thus settles the second open problem raised in \cite{wang2020long}. 
	Comparing with the state-of-the-art computationally efficient algorithms for episodic RL, e.g., \cite{azar2017minimax,zanette2019tighter,dann2019policy,jin2018q,zhang2020almost}, our algorithm enjoys an \emph{exponential} improvement in $H$.
	Comparing with the algorithm in \cite{wang2020long}, our algorithm is exponentially faster and achieves significantly better sample complexity in terms of $S,A,\epsilon$.
	See Table~\ref{tab:comparisons} for more detailed comparisons.
\end{enumerate*}

\begin{savenotes}
	\begin{table}[t]
		\centering
		\resizebox{1\columnwidth}{!}{%
			\renewcommand{\arraystretch}{2}
			\begin{tabular}{ |c|c|c|c|c|c|}
				\hline
				\textbf{Algorithm} & \textbf{Regret}  & \textbf{PAC Bound}& \makecell{\textbf{Poly} \\ \textbf{Time}} & \makecell{\textbf{Non-unif.}\\ \textbf{Reward}}&\makecell{\textbf{Log H}}  \\
				\hline
				\hhline{|=|=|=|=|=|=|}
				\makecell{UCBVI-BF \\
					\cite{azar2017minimax}} &    $ \widetilde{O}\left( \sqrt{SAK} + \sqrt{HK} + S^2AH\right)$&     $ \widetilde{O}\left(\frac{SA+H}{\epsilon^2} +\frac{S^2AH}{\epsilon}\right)$ &  Yes & No & No \\
				\hline 
				\makecell {UBEV \footnote{UBEV and ORLC provide a stronger result called mistake-stype PAC bounds. For more details, we refer readers to \cite{dann2019policy}. }\\
					\cite{dann2017unifying}
				}&     $ \widetilde{O}\left(\sqrt{SAH^2K}+S^2AH^2\right)$ &     $ \widetilde{O}\left(\frac{SAH^2}{\epsilon^2}+\frac{S^2AH^2}{\epsilon}\right)$ & Yes & No & No\\
				\hline 
				\makecell{UCB-Q-Bernstein\footnote{The model free-algorithms UCB-Q-Bernstein and UCBADV are for the inhomogeneous setting where $P_{1}(\cdot|s,a), P_2(\cdot|s,a),...,P_{H}(\cdot|s,a)$ are different.
				This difference necessarily incurs an additional $\sqrt{H}$ factor in the first term and an $H$ factor in the second term in regret.
				It is still an open problem whether a model-free algorithm can achieve a regret bound with the leading term scales $\widetilde{O}\left(\sqrt{SAK}\right)$.
					}\\
					\cite{jin2018q}
				} &     $ \widetilde{O}\left(\sqrt{SAH^2K}+ \sqrt{S^3A^3}H^{3}  \right)$ &     $ \widetilde{O}\left(\frac{SAH^2}{\epsilon^2}  +\frac{(SA)^{3/2}H^{3}}{\epsilon}\right)$ & Yes & No & No \\
				\hline 
				\makecell{ORLC \\ \cite{dann2019policy}}&    $ \widetilde{O}\left(\sqrt{SAK}+ S^2AH^2\right)$ &    $ \widetilde{O}\left(\frac{SA}{\epsilon^2}+\frac{S^2AH^2}{\epsilon}\right)$ & Yes & No & No\\
				\hline 
				\makecell{EULER\\
					\cite{zanette2019tighter}} & $ \widetilde{O}\left(\sqrt{SAK} + S^2AH+S^{3/2}AH^{3/2}\right)$&  $ \widetilde{O}\left(\frac{SA}{\epsilon^2} + \frac{S^2A +S^{3/2}AH^{3/2}}{\epsilon}\right)$ & Yes & Yes & No   \\
				\hline
				
				\makecell{UCBADV \\
					\cite{zhang2020almost}
				}         & $\widetilde{O}\left(\sqrt{SAHK} + S^2A^{3/2}H^{6}\right)$ &$ \widetilde{O}\left(\frac{SAH}{\epsilon^2}  +\frac{S^2A^{3/2}H^6}{\epsilon}\right)$ & Yes& No & No\\
				\hline
				
				\makecell{Trajectory Synthesis\\
					\cite{wang2020long}
				}
				&- &  $ \widetilde{O}\left(\frac{S^5A^4}{\epsilon^3}\right)$ & No & Yes & Yes \\
				\hline 
				\makecell{\algo \\This Work }& $ \widetilde{O}\left(\sqrt{SAK}+ S^2A\right)$ & $ \widetilde{O}\left(\frac{SA}{\epsilon^2} +\frac{S^2A}{\epsilon}\right)$ & Yes & Yes & Yes\\
				\hhline{|=|=|=|=|=|=|}
				\makecell{CB
					Lower Bound 
				}& $\Omega\left(\sqrt{SAK}\right)$ & $\Omega\left(\frac{SA}{\epsilon^2}\right)$  &- & - & - \\
				\hline
			\end{tabular}
		}
		\caption{
			Sample complexity comparisons for state-of-the-art episodic RL algorithms.
			See Section~\ref{sec:rel} for discussions on this table.
			$\widetilde{O}$ omits logarithmic factors.
			\textbf{Regret} and \textbf{PAC Bound} are sample complexity measures defined in Section~\ref{sec:pre}.
			\textbf{Non-unif. Reward}: Yes means the bound holds under Assumption~\ref{asmp:total_bounded} (allows non-uniformly bounded reward), and No means the bound only holds under Assumption~\ref{asmp:uniform}.
			\textbf{Poly Time}: Whether the algorithm runs in polynomial time.
			\textbf{Log H}: Whether the sample complexity bound depends logarithmically on $H$ instead of polynomially on $H$.
			\label{tab:comparisons}
		}
\vspace{-1cm}
	\end{table}
\end{savenotes}


Our algorithm and its analysis rely on the following new ideas.
\begin{enumerate*}
	\item We design a new exploration bonus based on Bernstein bound to ensure optimism.
	The key insight is that \emph{constants} in the bonus are crucial and helps maintain a monotonic property which helps propagates the optimism from level $H$ to level $1$.
	This property also leads a substantially simpler analysis than those in existing approaches.
	\item A crucial step in many UCB-based algorithm, including ours, is bounding the sum of variance of estimated value function across the entire planning horizon.
	Our technique is to use a  higher order expansion  to derive a recursive inequality that relates this sum to its higher moments.
	Importantly, this technique does not use any type of induction  from $H,H-1,\ldots,1$, which is used in most previous works and is the main technical barrier to obtain the logarithmic dependency on $H$.
\end{enumerate*}
See Section~\ref{sec:algo} and Section~\ref{sec:tac} for more technical expositions.

\section{Preliminaries}
\label{sec:pre}
\paragraph{Notations.}
Throughout this paper, we use $[N]$ to denote the set $\{1, 2, \ldots, N\}$ for $N\in \mathbb{Z}_{+}$. We use $\textbf{1}_{s}$ to denote the one-hot vector whose only non-zero element is in the $s$-th coordinate.
For an event $\mathcal{E}$, we use $\indict[\mathcal{E}]$ to denote the indicator function, i.e., $\indict[\mathcal{E}] = 1$ if $\mathcal{E}$ holds and $\indict[\mathcal{E}] = 0$ otherwise.
For notational convenience, we set $\iota  = \ln(2/\delta)$ throughout the paper.  
For two $n$-dimensional vectors $x$ and $y$, we use $xy$  to denote  $x^{\top}y$, use $\mathbb{ V}(x ,y) = \sum_{i}x_i y_i^2 - (\sum_{i}x_iy_i)^2 $. In particular, when $x$ is a probability vector, i.e., $x_i\geq 0$ and $\sum_i x_i = 1$, $\mathbb{ V}(x ,y) = \sum_i x_{i}\left( y_i - (\sum_{i}x_iy_i)\right)^2 = \min_{\lambda\in \mathbb{R}}\sum_i x_{i}\left( y_i -\lambda\right)^2$.
We also use $x^2$ to denote the vector $[x_1^2,x_2^2,...,x_n^2]^{\top}$ for $x = [x_1,x_2,...,x_n]^{\top}$. 
For two vectors $x,y$, $x \ge y$ denotes $x_i \ge y_i$ for all $i \in [n]$ and $x \le y$ denotes $x_i \le y_i$ for all $i \in [n]$.

\paragraph{Episodic Reinforcement Learning.}
A finite-horizon stationary Markov Decision Process (MDP) can be described by a tuple $\mdp =\left(\states, \actions, \trans ,R, H, \mu\right)$.
$\states$ is the finite state space with cardinality $S$.
$\actions$ is the finite action space with cardinality $A$.
$\trans: \states \times \actions \rightarrow \Delta\left(\states\right)$ is the transition operator which takes a state-action pair and returns a distribution over states.
$R : \states \times \actions \rightarrow \Delta\left( \mathbb{R} \right)$ is the reward distribution with a mean function $r:\states \times\actions \rightarrow \mathbb{R}$.
$H \in \mathbb{Z}_+$ is the planning horizon  (episode length).
 $\mu \in \Delta\left(\states\right)$ is the initial state distribution. 
 $\trans$, $R$ and $\mu$ are unknown.\footnote{
 	Some previous works consider the non-stationary MDP where $\trans$ and $R$ can vary on different $h \in [H]$~\citep{jin2018q,zhang2020almost}.
 	Non-stationarity will incur an $\sqrt{H}$ factor in the regret, which is necessary.
 	Transforming a regret bound for stationary MDP to that for non-stationary MDP  is often straightforward (with an additional $\sqrt{H}$ factor), but not vice-versa, because the main difficulty is how to effectively exploit the stationarity.
 	}
For notational convenience, we use $P_{s,a}$ and $P_{s,a,s'}$ to denote $P(\cdot|s,a)$ and $P(s'|s,a)$ respectively.

A policy $\pi$ chooses an action $a$ based on the current state $s \in \states$ and the time step $h \in [H]$. 
Note even though transition operator and the reward distribution are stationary, i.e., they do not depend on the level $h \in [H]$, the policy can be \emph{non-stationary}, i.e., at different level $h$, the policy can choose different actions for the same state.
Formally, we define $\pi = \{\pi_h\}_{h = 1}^H$ where for each $h \in [H]$, $\pi_h : \states \to \actions$ maps a given state to an action.
The policy $\pi$ induces a (random) trajectory $\{s_1,a_1,r_1,s_2,a_2,r_2,\ldots,s_{H},a_{H},r_{H} \}$,
where $s_1 \sim \mu$, $a_1 = \pi_1(s_1)$, $r_1 \sim R(s_1,a_1)$, $s_2 \sim \trans(\cdot|s_1,a_1)$, $a_2 = \pi_2(s_2)$, etc.

Our target is to find a policy $\pi$ that maximizes the expected total
reward, i.e.
$
\max_\pi \expect \left[\sum_{h=1}^{H} r_h \mid \pi\right] 
$
where the expectation is over the initial distribution state $\mu$, the transition operator $P$ and the reward distribution $R$.
As for scaling, we make the following assumption about the reward.
 As we will discuss in Section~\ref{sec:rel}, this is a more general assumption than the assumption often made in most previous works.
 \begin{assumption}[Bounded Total Reward]\label{asmp:total_bounded}
 	The reward satisfies that $r_h\geq 0$ for all $h\in [H]$. Besides, for all policy $\pi$, $\sum_{h=1}^H r_h\leq 1$ almost surely.
 \end{assumption}
\noindent Given a policy $\pi$, a level $h \in [H]$ and a state-action pair
$(s,a) \in \states \times \actions$, the $Q$-function is defined as:
$
Q_h^\pi(s,a) = \expect\left[\sum_{h' = h}^{H}r_{h'}\mid s_h =s, a_h = a, \pi\right].
$
Similarly, given a policy $\pi$, a level $h \in [H]$, the value function of a given state
$s \in \states$ is defined as: 
$
V_h^\pi(s)=\expect\left[\sum_{h' = h}^{H}r_{h'}\mid s_h =s,
  \pi\right].
$
Then Bellman equation establishes the following identities for policy $\pi$ and $(s,a,h) \in \states \times \actions \times [H]$:
 $
Q_h^\pi(s,a) =r(s,a)+ P_{s,a}^{\top}V_{h+1}^\pi$ and $V_{h}^\pi(s) = \max_{a}Q_{h}^\pi(s,a).
 $
 Throughout the paper, we let $V_{H+1}(s) = 0$ and $Q_{H+1}(s,a) = 0$ for notational simplicity.
We use $Q^*_h$ and $V^*_h$ to denote the optimal $Q$-function and $V$-function at level $h \in [H]$, which satisfies for any state-action pair $(s,a) \in \states \times \actions$, $Q^*_h(s,a) = \max_{\pi}Q^{\pi}_h(s,a)$ and $V^*_h(s) =\max_{\pi}V^{\pi}_h(s)$.
%

 When $H=1$, the episodic RL reduces to the problem of finding a policy $\pi: \states \rightarrow \actions$ that maximizes the expected reward $
 \max_{\pi} \expect_{s \sim \mu(\cdot), r_{\cb}\sim R(s,\pi(s))}\left[r_{\cb}\right].
 $
 This is called the contextual bandit (CB) problem.
 RL is  more difficult than CB as we also need to deal with the long planning horizon $H$ and transition operator $P$, which are absent in CB.
 In this paper, we investigate whether the these two ingredients incur additional hardness in terms of the sample complexity.

 \paragraph{Sample Complexity.}
In this paper we use two measures to quantify sample complexity.
The agent interacts with the environment for $K$ episodes, and it chooses a policy $\pi^k$ at the $k$-th episode.
The total regret is
\[
\mathrm{Regret}(K) =  \sum_{k=1}^K V_1^*(s_1^k) - V_1^{\pi^k}(s_1^k).
\]

PAC-RL sample complexity is another measure which counts the number of episodes to find an $\epsilon$-optimal policy $\pi$, i.e., \[\expect_{s_1 \sim \mu}\left[V_1^*(s_1) - V^\pi(s_1)\right] \le \epsilon.\]
As pointed out in \cite{jin2018q}, suppose that one has an algorithm that achieves $CK^{1-\alpha}$ regret for some $\alpha \in (0,1)$ and some $C$ independent of $T$, by randomly selecting from policy $\pi^k$ used in $K$ episodes, $\pi$ satisfies $\expect_{s_1 \sim \mu}\left[V_1^*(s_1) - V^\pi(s_1)\right] = O\left(CK^{-\alpha}\right)$.
This reduction is often near-optimal to obtain PAC-RL sample complexity guarantee.
On the other hand, there is no general near-optimal reduction that transform a PAC-RL bound to a regret bound.

%
%
%
%

\section{Background and Related Work}
\label{sec:rel}


We mostly focus on papers that are for the episodic RL setting described in Section~\ref{sec:pre}.
A summary of the most relevant previous regret and PAC bounds, together with the results proved in this paper is provided in Table~\ref{tab:comparisons}.
We remark that there are also related settings, e.g., infinite-horizon discounted MDP, weakly-communicating MDP, learning with a generative model, etc.
These settings are beyond the scope of this paper
, though our  techniques may be also applied to these settings.

 \paragraph{Reward Assumption.}
 In episodic tabular RL, the sample complexity depend on $\abs{\states}$, $\abs{\actions}$ and $H$, all of which are assumed to be finite. 
 For the reward, the widely adopted assumption is $r_h\in [0,1]$ for all $h\in [H]$, which implies the total reward $\sum_{h=1}^H r_h\in [0,H]$. 
 To have a fair comparison with CB and illustrate the hardness due to the planning horizon and/or unknown transition operator, one should scale down the reward by an $H$ factor such that the total reward is bounded in $[0,1]$. \footnote{When comparing with existing algorithms, we also scale down their bounds by an $H$ factor.}
 This leads to the following assumption.
 \begin{assumption}[Uniformly Bounded Reward]\label{asmp:uniform}
$r_h\in [0,1/H]$ for all $h\in [H]$.
 \end{assumption}

 Clearly, Assumption~\ref{asmp:total_bounded} is more general than Assumption~\ref{asmp:uniform}, so any upper bound under
 Assumption~\ref{asmp:total_bounded}, also implies an upper bound under Assumption~\ref{asmp:uniform}.
 From practical point of view, as argued in \cite{jiang2018open}, since  environments under Assumption~\ref{asmp:total_bounded} can have one-step reward as high as a constant,
 Assumption~\ref{asmp:total_bounded} is more natural in environments
 with sparse rewards, which are often considered to be hard.
 From a theoretical point view, to design provably efficient algorithms under Assumption~\ref{asmp:total_bounded} is more difficult, as one needs to consider a more global structure.
 \footnote{Under Assumption~\ref{asmp:total_bounded}, the reward still satisfies $r_h \in [0,1]$, so if an algorithms enjoys an sample complexity bound under Assumption~\ref{asmp:uniform}, scaling up this bound by an $H$ factor for regret  or $H^2$ for PAC bound,  one can obtain a bound under Assumption~\ref{asmp:total_bounded}.
 However, this reduction is suboptimal in terms of $H$, so we display their original results and add a column indicating whether the bound is under Assumption~\ref{asmp:total_bounded} or Assumption~\ref{asmp:uniform}.}
The sample complexity bounds in this paper hold under the more general Assumption~\ref{asmp:total_bounded}.

\paragraph{Previous Sample Complexity Bounds.}
There is a long list of sample complexity guarantees for episodic tabular RL~\citep{kearns2002near,brafman2002r,kakade2003sample,strehl2006pac,strehl2008analysis,kolter2009near,bartlett2009regal,jaksch2010near,szita2010model,lattimore2012pac,osband2013more,dann2015sample,azar2017minimax,dann2017unifying,osband2017posterior,agrawal2017optimistic,jin2018q,fruit2018near,talebi2018variance,dann2019policy,dong2019q,simchowitz2019non,russo2019worst,zhang2019regret,cai2019provably,zhang2020almost,yang2020q,pacchiano2020optimism,neu2020unifying}.
There are two popular types of algorithms, model-based algorithms and model-free algorithms.
In episodic RL, model-based algorithms' space complexity scales quadratically with $S$ and model-free algorithms and model-free algorithms' space complexity linearly with $S$.
Both types of algorithms often rely on using UCB to ensure optimism and guide exploration.
Under Assumption~\ref{asmp:uniform}, both the state-of-the-art model-based and model-free algorithms  achieve regret bounds of the form $\widetilde{O}\left(\sqrt{SAK} + \poly\left(SAH\right)\right)$.
Recently, \citet{zanette2019tighter} proposed a model-based algorithm which achieves the regret of the same form under Assumption~\ref{asmp:total_bounded}.
The first term in these bounds matches the lower bound, $\Omega\left(\sqrt{SAK}\right)$  up to logarithmic factors~\citep{bubeck2012regret,dann2015sample,osband2016on}.
See Table~\ref{tab:comparisons} for specific bounds in these works and other related ones.

These bounds become non-trivial (regret bound sub-linear in $K$ or PAC bound smaller than $1$) only when $K \gg H$ or $\epsilon \ll \frac{1}{H}$.
However, as explained in \cite{jiang2018open}, in many scenarios with
a long planning horizon such as control, this regime is not interesting,  and the more interesting regime is when $K \ll H$ or $\epsilon \gg 1/H$.

The recent work by \citet{wang2020long} bypassed this barrier via a completely different approach and obtained an $\widetilde{O}\left(\frac{S^5A^4}{\epsilon^3}\right)$ PAC-RL sample complexity bound, which is the first bound that scales logarithmically with $H$.
They built an $\epsilon$-net over for optimal policies and designed a simulator to evaluate all policies within the $\epsilon$-net.
However, their algorithm runs in exponential time and its sample complexity's dependencies on $S$, $A$, $\epsilon$ are far from optimal.
Furthermore, their work does not rule out the possibility that long planning horizon and/or unknown state-dependent transitions force the agent acquire more samples than CB in terms of $S$ and $A$ to learn a near-optimal policy.

In this work, we follow the conventional UCB-based approach. 
Our algorithm is computationally efficient and achieves $\widetilde{O}\left(\sqrt{SAK}+S^2A\right)$ regret and $\widetilde{O}\left(\frac{SA}{\epsilon^2} +\frac{S^2A}{\epsilon} \right)$ PAC-RL bound, which outperform all existing sample complexity bounds, including the additive terms.
See Table~\ref{tab:comparisons} for more detail.



\begin{algorithm}[t]
	\caption{\textbf{M}onotonic \textbf{V}alue \textbf{P}ropagation (MVP) \label{alg1}	}
	\begin{algorithmic}[1]
		\STATE{ \textbf{Input:} Trigger set $\mathcal{L} \leftarrow \{ 2^{i-1}| 2^{i}\leq KH, i=1,2,\ldots \}$. $c_1 = \frac{460}{9} ,c_2= 2\sqrt{2},c_3 =\frac{544}{9}$. }
		\FOR{$(s,a,s',h)\in \mathcal{S}\times \mathcal{A}\times\mathcal{S}\times[H]$}
		\STATE{ $N(s,a)\leftarrow 0$; $\theta(s,a)\leftarrow 0$;  $n(s,a)\leftarrow 0$; }
		\STATE{ $N(s,a,s')\leftarrow 0$; $\hat{P}_{s,a,s'}\leftarrow 0$, $Q_h(s,a)\leftarrow 1$; $V_h(s)\leftarrow 1$.}
		\ENDFOR
		\FOR {$k=1,2,...$}
		\FOR {$h=1,2,...,H$}
		\STATE{ Observe $s_{h}^k$;}
		\STATE{ Take action $ a_h^k= \arg\max_{a}Q_h(s_h^k,a)$;} \label{line:choose_action}
		\STATE{ Receive reward $r_h^k$ and observe $s_{h+1}^k$.}
		\STATE{ Set $(s,a,s',r)\leftarrow (s_h^k,a_h^k,s_{h+1}^k,r_h^k)$;.}
		\STATE{ Set $N(s,a) \leftarrow  N( s,a )+1$, \, $\theta(s,a)\leftarrow \theta(s,a)+r$, \,$N(s,a,s') \leftarrow   N(s,a,s')+1$.}
		\STATE{ \verb|\\| \emph{Update empirical reward and transition probability}}
		\IF {$N(s,a)\in \mathcal{L}$}  \label{line:rp_update_start}
		\STATE{  Set $\hat{r}(s,a)\leftarrow  \mathbb{I}\left[N(s,a)\geq 2\right]\frac{2\theta(s,a)}{N(s,a)} + \mathbb{I}\left[N(s,a)=1\right]\theta(s,a)$ and $\theta(s,a)\leftarrow 0$.}
		\STATE Set $\hat{P}_{s,a,\tilde{s}} \leftarrow  N(s,a,\tilde{s}) /N(s,a)$ for all $\tilde{s} \in \states$.
		\STATE{ Set $n(s,a)\leftarrow N(s,a)$;}
		\STATE{ Set TRIGGERED = TRUE.}
		\ENDIF \label{line:rp_update_end}
		\ENDFOR
		\STATE{ \verb|\\| \emph{Update $Q$-function}}
		\IF {TRIGGERED}
		\FOR{$h=H,H-1,...,1$}
		\FOR{$(s,a)\in \mathcal{S}\times \mathcal{A}$}
		\STATE 	 {		
			Set
			\begin{align} 
			~~~~~~~~~~~~~&b_h(s,a)\leftarrow c_1 \sqrt{\frac{   \mathbb{ V}(\hat{P}_{s,a} ,V_{h+1}) \iota  }{ \max\{n(s,a),1 \} }}+ c_{2}\sqrt{ \frac{ \hat{r}(s,a)\iota }{\max\{n(s,a),1 \} } }+c_3\frac{\iota}{ \max\{n(s,a) ,1\}  },  \label{equpdate1}  
			\\ 	\hspace{-20ex} 	& Q_h(s,a)\leftarrow \min\{    \hat{r}(s,a)+\hat{P}_{s,a} V_{h+1} +b_h(s,a)    ,1\}, \label{equpdate2}
			\\ & V_{h}(s) \leftarrow \max_{a}Q_h(s,a).\nonumber
			\end{align}
			\vspace{-3ex}
		}
		\ENDFOR
		\ENDFOR
		\STATE{ Set TRIGGERED = FALSE}
		\ENDIF
		\ENDFOR
	\end{algorithmic}
\end{algorithm}

\section{Main Algorithm}
\label{sec:algo}

In the section, we introduce the \textbf{M}onotonic \textbf{V}alue \textbf{P}rorogation (MVP)  algorithm. 
The pseudo code is listed in Algorithm~\ref{alg1}.
The algorithm adopts the doubling update framework proposed in \cite{jaksch2010near}.
More precisely, we define a trigger set $\mathcal{L} = \{ 2^{i-1} |2^{i-1}\leq KH, i=1,2,\ldots\}$.
The algorithm proceeds through epochs where each epoch ends whenever there exists a state-action pair $(s,a)$ such that  the number of visits of $(s,a)$ falls into $\mathcal{L}$.  
In each epoch, we use the same policy induced by the current estimation of $Q$-function (cf. Line~\ref{line:choose_action}).

We update the empirical reward and transition probability of a state-action pair $(s,a)$ only when the number of visits of $(s,a)$ falls into $\mathcal{L}$. (cf.  Line~\ref{line:rp_update_start}).
For the transition probability, we use the standard maximum likelihood estimation.
For the reward function, we only use the data collected in the current epoch to calculate the empirical reward.
This will simplify the analysis and save a log factor.
See Lemma~\ref{lemma_bdb} and its proof for more detail.


If in an episode, we update the reward and the transition probability of state-action pair, we will also update the $Q$-function estimation at the end of this episode.
We define the bonus in Equation~\eqref{equpdate1} and our optimistic estimator of $Q$-function  in Equation~\eqref{equpdate2}.
Note our bonus function only contains three terms.
The first term and the third term correspond to the upper confidence bound of transition and the second term corresponds to the upper confidence bound of the reward.
The main novelty is that by setting appropriate $c_1,c_2,c_3$, the optimism can propagate from level $H$ to level $1$ without adding additional terms.
We emphasize all previous results that can achieve $O\left(\sqrt{SAK}\right)$ as the first term in the regret bound (cf. Table~\ref{tab:comparisons}) require more sophisticated bonus constructions.
See Section~\ref{sec:tac} for more technical explanations. 



\section{Technique Overview}
\label{sec:tac}
An optimistic algorithm needs to guarantee that (with high probability) the estimated $Q$-function is always an upper bound of the optimal $Q$-function, i.e., $Q_h(s,a) \geq Q_h^*(s,a)$ for all $(s,a,h) \in \states \times \actions \times \times [H]$.
Note this also implies $V_h(s) \ge V^*_h(s)$.
\footnote{
	In this section, we drop the dependency on $k$ for the ease of presentation.
	}
model-based algorithms, including ours, use the following estimator for the $Q$-function
\begin{align}
Q_h(s,a) = \hat{r}(s,a)+ \hat{P}_{s,a}V_{h+1} + b_h(s,a) \label{eq_sec1_1}
\end{align}
where $b_h$ is the bonus to guarantee $Q_h$ is an upper bound of $Q^*$.
The main difference among algorithms is the choice of $b_h$.
In the following, we first review existing approaches in constructing $b_h$ and why they failed to obtain the logarithmic dependency on $H$. 
Then we introduce our construction of $b_h$ and the corresponding analysis to overcome the barrier.

\paragraph{Main Difficulty.}
%
%
Fix a level $h$.
Suppose the estimator for level $h+1$ satisfies $Q_{h+1} \ge Q_{h+1}^*$, and this implies $V_{h+1} \ge V_{h+1}^*$.
Many previous optimistic algorithms use the following induction strategy to construct the bonus for level $h$:
\begin{align}
Q_h(s,a) =& \hat{r}(s,a)+ \hat{P}_{s,a}V_{h+1} + b_h(s,a)  \nonumber
\\ & \geq  \hat{r}(s,a) + \hat{P}_{s,a}V^*_{h+1}+b_h(s,a)  \label{eq_loose_ineq}
\\ & = Q_h^*(s,a) + (\hat{P}_{s,a}-P_{s,a})V_{h+1}^* + \left(\hat{r}(s,a) - r(s,a)\right)+ b_h(s,a), \label{eq_sec1_2}
\end{align}
where the inequality~\eqref{eq_loose_ineq} follows from the induction hypothesis $V_{h+1} \ge V^*_{h+1}$ and the last equality follows from Bellman equation.
To ensure optimism, existing works design $b_h(s,a)$ to be an upper bound of $(\hat{P}_{s,a}-P_{s,a})V_{h+1}^* + \left(\hat{r}(s,a) - r(s,a)\right)$ using concentration inequalities.

The tricky part is in bounding $\left(\hat{P}_{s,a}-P_{s,a}\right)V_{h+1}^*$.
As discussed in \cite{azar2017minimax}, since one does not know $V_{h+1}^*$, one has to replace $V_{h+1}^*$ by its estimation $V_{h+1}$ and introduce \emph{additional terms}  in  $b_h(s,a)$ to ensure optimism.
This approach has been used in all previous  approaches whose regret bounds' first term is $\widetilde{O}\left(\sqrt{SAK}\right)$
~\citep{azar2017minimax,dann2019policy,zanette2019tighter,zhang2020almost}.

Unfortunately, the regret induced by the additional terms lead to (at least) a linear dependency on $H$ because in the analyses, one needs to make $\|V_{h+1}-V_{h+1}^*\| = O\left(\frac{\epsilon}{H}\right)$ so that the final error is $O\left(\epsilon\right)$ (via e.g., performance difference lemma~\citep{kakade2003sample}).
To make $\|V_{h+1}-V_{h+1}^*\| = O\left(\frac{\epsilon}{H}\right)$, the sample complexity needs to scale at least linearly with $H$.


\paragraph{Technique 1:  Monotonic Value Propagation.}
In this work, we do not go through inequality~\eqref{eq_loose_ineq} in constructing the bonus.
Our main strategy is to view $Q_h$ as a function of the variable $V_{h+1}$ (cf. Equation~\eqref{eq_sec1_1}), which we denote as $  Q_{h}(V_{h+1})$ and we design $b_h$ such that the function $Q_h(\cdot)$ satisfies two principles:\footnote{$b_h$ can depend on $V_{h+1}$ as well.}
\begin{itemize}
\item \textbf{Optimism}: $Q_h(V_{h+1}^*) \ge Q_h^*$;
\item \textbf{Monotonicity}: For two variables $V_{h+1}$ and $V_{h+1}'$ with $V_{h+1}\ge V_{h+1}'$, $Q_h(V_{h+1}) \ge Q_h(V_{h+1}')$.
\end{itemize}
If our estimation on $Q$ function satisfies these two properties, under the induction hypothesis that $V_{h+1} \ge V_{h+1}^*$, we have \begin{align*}
	Q_h(V_{h+1}) \ge Q_h(V_{h+1}^*) \ge Q_h^*.
\end{align*}
While the first principle, optimism, is adopted in most previous algorithms,
the second monotonicity principle is new in the literature and we believe this idea can be useful in algorithm design for other RL problems.

Now we instantiate this idea.
Recall our estimator defined in Equation~\eqref{equpdate1}-\eqref{equpdate2}\[
Q_h(s,a) \triangleq  \min\left\{   \hat{r}(s,a)+\hat{P}_{s,a} V_{h+1} +c_1 \sqrt{\frac{   \mathbb{ V}(\hat{P}_{s,a} ,V_{h+1}) \iota  }{ \max\{n(s,a),1 \} }}+ c_{2}\sqrt{ \frac{ \hat{r}(s,a)\iota }{\max\{n(s,a),1 \} } }+c_3\frac{\iota}{ \max\{n(s,a) ,1\} }  ,1\right\}.
 \]
The optimism principle can be easily implemented using empirical Bernstein inequality (see Lemma \ref{empirical bernstein}).
For the monotonicity principle, we will carefully tune the \emph{constants} $c_1,c_2,c_3$.
See Lemma~\ref{lemma1} for more details.
\footnote{As will be clear in our proof, our actual estimator of $Q$-function satisfies that $Q_h
\geq F_h$ for some function $F_h$, and $F_h$ satisfies the two principles mentioned above. 
We do not discuss this subtlety in detail for the ease of presentation.
}


\paragraph{Technique 2: Bounding the Total Variance via Recursion}
Using a sequence of fairly standard steps in the literature, we can bound the regret by the square-root of  the total variance $\sqrt{\sum_{h=1}^H \mathbb{V}(P_{s_h,a_h}, V_{h+1}^k) }$ along with some other lower order terms.
To explain our high-level idea, we present analysis for the total variance in a single episode with estimated value function replaced by the true value function, i.e.,$ \sum_{h=1}^H \mathbb{V}(P_{s_h,a_h}, V_{h+1}^* )$
\begin{align}
& \sum_{h=1}^H \mathbb{V}(P_{s_h,a_h}, V_{h+1}^* ) =\sum_{h=1}^H \left( P_{s_h,a_h} (V_{h+1}^*)^2 - (P_{s_h,a_h}V_{h+1}^*)^2   \right)  \nonumber
 \\&\quad   = \sum_{h=1}^H \left(P_{s_h,a_h} (V_{h+1}^*)^2 -(V_{h+1}^*(s_{h+1}))^2  \right) +\sum_{h=1}^H \left( (V_h^*(s_h))^2 -(P_{s_h,a_h}V_{h+1}^*)^2    \right) - (V_1^*(s_1))^2 \nonumber
 \\ & \quad \leq  \sum_{h=1}^H \left(P_{s_h,a_h} (V_{h+1}^*)^2 -(V_{h+1}^*(s_{h+1}))^2  \right)  +2 \sum_{h=1}^H \left(V_h^*(s_h)-Q_h^*(s_h,a_h) \right) +2\sum_{h=1}^H r(s_h,a_h)\nonumber
 \\ & \quad \leq  \sum_{h=1}^H \left(P_{s_h,a_h} (V_{h+1}^*)^2 -(V_{h+1}^*(s_{h+1}))^2  \right)  +2 \sum_{h=1}^H \left(V_h^*(s_h)-Q_h^*(s_h,a_h) \right) +2
 \\ & \quad \leq  \tilde{O}\left( \sqrt{\sum_{h=1}^H \mathbb{V}(P_{s_h,a_h}, (V^*_{h+1})^2 )   }  + \sum_{h=1}^H \left(V_h^*(s_h)-Q_h^*(s_h,a_h) \right) \right). \label{eq_sec1_0}
\end{align}
where the first inequality we dropped $V_1^*(s_1)$, the second inequality we used the total reward is bounded by $1$ and the last step holds with high probability due to a simple corollary of Freedman's inequality ~\citep{freedman1975tail} (see Lemma~\ref{self-norm}).
%

We can roughly view the second term in \eqref{eq_sec1_0} as the regret in this episode.
Therefore, Inequality~\eqref{eq_sec1_0} shows the total variance can be bounded by the square-root of the total variance of the \emph{second moment}  and the regret.
We then apply this argument \emph{recursively}, i.e., $m\geq 1,2,\ldots$, we can bound the total variance of the $2^m$-th moment $\sum_{h=1}^H\mathbb{V}(P_{s,a},(V^*_{h+1})^{2^m} )$ by $\sum_{h=1}^H\mathbb{V}(P_{s,a},(V^*_{h+1})^{2^{m+1}} )$ and the regret.
Also note that $\sum_{h=1}^H\mathbb{V}(P_{s,a},(V^*_{h+1})^{2^m} )$ is bounded by $H$ almost surely for any $m$.

Based on the basic lemma below,  we can obtain a $\poly\log H$ bound for $\sum_{h=1}^H \mathbb{V}(P_{s_h,a_h}, V_{h+1}^*)$.

\begin{lemma}\label{lemma2}
Let $\lambda_1,\lambda_2,\lambda_4\geq 0$, $\lambda_3\geq 1$ and $i' =\log_2(\lambda_1)$.	Let $a_{1},a_{2},...,a_{i'}$ be non-negative reals such that $a_{i}\leq \lambda_{1}$ and $a_{i}\leq \lambda_{2}\sqrt{a_{i+1}+ 2^{i+1} \lambda_3 } +\lambda_4$ for any $1\leq i\leq i'$. 
	 Then we have that $a_{1}\leq \max\{ (\lambda_2 +\sqrt{\lambda_2^2+\lambda_4} )^2  ,\lambda_{2}\sqrt{8\lambda_3}  +\lambda_4 \}$   .
\end{lemma}

\section{Proof Sketch of Theorem \ref{thm1}}
\label{sec:main_proof}
In this section, we present the proof sketch of Theorem \ref{thm1}. We first introduce a few notations: we use \qkhsa, $V_h^k(s)$ and $\hat{P}^k_{s,a}$ to denote the values of  $Q_h(s,a)$, $V_h(s)$ and $\hat{P}_{s,a}$ in the beginning of the $k$-th episode. 
Let $n^k(s,a), b_h^k(s,a)$ and $\hat{r}_h^k(s,a)$ denote the value of $\max\{n(s,a) ,1\}$, $b_h(s,a)$ and $\hat{r}(s,a)$  in \eqref{equpdate1} used for computing $Q_h^k(s,a)$.Lastly, we define $V_h^k = [V_h^k(s)]^{T}_{s\in \mathcal{S}}$ for convenience.

\subsection{Proof of Optimism}\label{sec:main_proof_1}
We define $\mathcal{E}_{1}$ to be the event where 
\begin{align}	
\left|(\hat{P}^k_{s,a}-P_{s,a})V_{h+1}^*\right| \leq  2\sqrt{\frac{ \mathbb{ V}(\hat{P}^k_{s,a}, V_{h+1}^*  )\iota }{ n^k(s,a)}} +\frac{14\iota}{3n^k(s,a)} \label{eq_lemma1_0}
\end{align}
holds for all $(s,a,h,k) \in \states \times \actions \times [H] \times [K]$.  
We also define $\mathcal{E}_2$ be the event where
\begin{align}
	\left|\hat{r}_h^k(s,a)-r(s,a)\right|\leq 2\sqrt{\frac{2\hat{r}_h^k(s,a)\iota }{n^k(s,a)}} +\frac{28\iota}{3n^k(s,a)}\label{eq_lemma1_0.5}
\end{align}
holds for any possible $(s,a,h,k) \in \states \times \actions \times [H] \times [K]$.

The following lemma shows $\mathcal{E}_1$ and $\mathcal{E}_2$ hold with high probability.
The analysis will be done assuming the successful event $\mathcal{E}_{1}\cap \mathcal{E}_2$ holds in the rest of this section.
\begin{lemma}\label{lemma:high_prob_e1_e2}
$\mathbb{P}\left[\mathcal{E}_1\cap \mathcal{E}_2\right]\geq 1 -2SA(\log_{2}KH+1)\delta$. 
\end{lemma}
 By our exploration bonus, the $Q$-function is always optimistic with high probability.

\begin{lemma}\label{lemma1}
	Conditioned on $\mathcal{E}_1\cap \mathcal{E}_2$, $Q_h^k(s,a)\geq Q^*_h(s,a)$ for all $(s,a,h,k) \in \states \times \actions \times [H] \times [K]$
\end{lemma}


\subsection{Bounding the Bellman Error}\label{sec:main_proof_2}
When the $Q$-function is optimistic, the major term in the regret of the induced policy is the sum of the Bellman error (see Lemma~\ref{lemma:add1}). So we start with a simple bound for the Bellman error induced by the $Q$-function. 
\begin{lemma}\label{lemma_be}  With probability $1-3S^2AH(\log_2(KH)+1)\delta$, for any $1\leq k\leq K$, $1\leq h\leq H$ and $(s,a)$, it holds that
	\begin{align}
	& \quad Q_h^k(s,a)-r(s,a)-P_{s,a}V_{h+1}^k \nonumber
	\\ & \leq \min\{ 2b_h^k(s,a) +c_4 \sqrt{\frac{ \mathbb{ V} (P_{s,a},V^*_{h+1}) \iota  }{ n^k(s,a) }} +c_{5}\sqrt{  \frac{S  \mathbb{ V}(P_{s,a}, V^k_{h+1}-V^*_{h+1}  ) \iota }{ n^k(s,a)} } +c_6\frac{S\iota}{n^k(s,a)},1 \}\label{eq_be_0}
	\end{align}
	for some large enough universal constants $c_4, c_5$ and $c_6$.
\end{lemma}

In the rest of this section, we let $\beta_{h}^k(s,a)$ be a shorthand of \text{RHS} of \eqref{eq_be_0}, i.e.,
\begin{align}
\beta_h^k(s,a):= \max\{ 2b_h^k(s,a) +c_4 \sqrt{\frac{ \mathbb{ V} (P_{s,a},V^*_{h+1}) \iota  }{ n^k(s,a) }} +c_{5}\sqrt{  \frac{S  \mathbb{ V}(P_{s,a}, V^k_{h+1}-V^*_{h+1}  ) \iota }{ n^k(s,a)} } +c_6\frac{S\iota}{n^k(s,a)},1 \}.\label{eq_beta}
\end{align}
We further define $\tilde{Q}^k_h(s,a): = Q_h^k(s,a)-Q_{h}^*(s,a)$, $\tilde{V}_h^k(s)=V_h^k(s)-V_{h}^*(s)$ and $\tilde{V}_{h}^k =[\tilde{V}_h^k(s)]^{\top}_{s\in \mathcal{S}}$, so  by Lemma \ref{lemma_be} and Bellman equation $Q_h^*(s,a)=r(s,a)+P_{s,a}V_{h+1}^*$, we have that with probability $1-3S^2AH(\log_2(KH)+1)\delta$, for all $(s,a,h,k) \in \states \times \actions \times [H] \times [K]$:
\begin{align}
\tilde{V}_h^k(s_h^k)-P_{s_h^k,a_h^k}\tilde{V}_{h+1}^k \leq \tilde{Q}_h^k(s_h^k,a_h^k)-P_{s_h^k,a_h^k}\tilde{V}_{h+1}^k \leq \beta_{h}^k(s,a).\label{eq_sec3_1}
\end{align}

\subsection{Regret Analysis}\label{sec:main_proof_3}
Let $\mathcal{K}$ be the set of indexes of episodes in which no update is triggered. By the update rule, it is obvious that  $|\mathcal{K}^{C}|\leq SA(\log_2(KH)+1)$.
Let $h_0(k)$ be is the first time an update is triggered in the $k$-th episode if there is an update in this episode and otherwise $H+1$. Define $\mathcal{X}_0 = \{(k,h_0(k))| k\in \mathcal{K}^{C} \}$ and
$\mathcal{X} = \{(k,h) | k\in \mathcal{K}^{C}, h_0(k)+1\leq h\leq H \}$.

Then we define $\check{V}_h^k(s_h^k,a_h^k) = \mathbb{I}\left[(k,h)\notin \mathcal{X}\right]\cdot V_h^k(s_h^k,a_h^k)$. We also set $\check{\beta}_h^k(s_h^k,a_h^k) = \mathbb{I}\left[(k,h)\notin \mathcal{X}\right]\cdot \beta_h^k(s_h^k,a_h^k)$ and  $\check{r}_h^k= \mathbb{I}\left[(k,h)\notin \mathcal{X}\right]\cdot r(s_h^k,a_h^k)$. By Lemma \ref{lemma_be}, we have that with probability $1-3S^2AH(\log_2(KH)+1)\delta$ ,
\begin{align}
\check{V}_h^k(s_h^k,a_h^k)\leq \check{r}_h^k+\check{\beta}_h^k(s_h^k,a_h^k) + P_{s,a}\check{V}^k_{h+1},\label{eq:sec2_0.5}
\end{align}
 for any $(h,k)\notin \mathcal{X}_0$ and 
\begin{align}
\check{V}_h^k(s_h^k,a_h^k)\leq \check{r}_h^k +\check{\beta}_h^k(s_h^k,a_h^k) + P_{s,a}\check{V}^k_{h+1}+1,\label{eq:sec2_0.7}
\end{align}
for any $(h,k)\in \mathcal{X}_0$.
\begin{remark}
It is hard to analyze the regret in the episodes not in $\mathcal{K}$ directly since $\mathbb{I}[k\in \mathcal{K}]$ is not measurable in $\mathcal{F}_1^k$. Instead, we introduce $\mathcal{X}$ and analyze the regret in the steps not in $\mathcal{X}$ because  $\mathbb{I}[(k,h)\notin \mathcal{X}]$ is measurable in $\mathcal{F}_h^k$.
\end{remark}
By Lemma \ref{lemma1} and \ref{lemma_be}, we have that
\begin{lemma}\label{lemma:add1}
With probability at least $1-5S^2AH(\log_2(KH)+1)\delta$,
    \begin{align}
&\text{Regret}(K) := \sum_{k=1}^K \left(V^*_1(s_1^k) -V_1^{\pi^k}(s_1^k) \right) \nonumber
\\ & \quad  \leq  \sum_{k=1}^K\sum_{h=1}^H  (P_{s_h^k,a_h^k}-\textbf{1}_{s_{h+1}^k})\check{V}_{h+1}^k+ \sum_{k=1}^K\sum_{h=1}^H \check{\beta}_h^k(s_h^k,a_h^k) + \sum_{k=1}^K (\sum_{h=1}^H \check{r}_h^k- V^{\pi^k}_1(s_1^k)) +|\mathcal{K}^{C}|.\label{eq_sec2_1}
\end{align}
\end{lemma}
Define $M_1 =  \sum_{k=1}^K\sum_{h=1}^H  (P_{s_h^k,a_h^k}-\textbf{1}_{s_{h+1}^k}) \check{V}_{h+1}^k$, $M_2 =\sum_{k=1}^K\sum_{h=1}^H \check{\beta}_h^k(s_h^k,a_h^k) $ and $M_3 =  \sum_{k=1}^K (\sum_{h=1}^H \check{r}_{h}^k- V^{\pi^k}_1(s_1^k))$. We will bound these three terms separately by the lemmas below.

\begin{lemma}\label{lemma:bound_M1}
\begin{align}
\mathbb{P}\left[     |M_1| > 2\sqrt{2 \sum_{k=1}^K\sum_{h=1}^H \mathbb{V}( P_{s_h^k,a_h^k},\check{V}_{h+1}^k )\iota } +6\iota   \right]\leq 2(\log_{2}(KH)+1)\delta.
\end{align}
\end{lemma} 
 
\begin{lemma}\label{lemma:bound_M2}
Define $i_{\mathrm{max}} = \max\{i|2^{i-1}\leq KH\} = \left\lfloor \log_2(KH) \right\rfloor +1$.    With probability $1-\big( 6S^2AH(\log_2(KH)+1)+ 6(\log_2(KH)+1)\log_2(H) \big) \delta$,
\begin{align}
M_2& \leq O\left(  \sqrt{SAKi_{\mathrm{max}}\iota} + \sqrt{S^2Ai_{\mathrm{max}}\sqrt{M_2}\iota^{3/2}  } +\sqrt{SAi_{\mathrm{max}}K\iota }+ S^2A\iota\log_2(KH)       \right)\nonumber
\\& \leq  O\left( \sqrt{SAKi_{\mathrm{max}}\iota}  +   S^2A\iota\log_2(KH)        \right).
\end{align}
    
\end{lemma}

\begin{lemma}\label{lemma:bound_M3}
    \begin{align}
   \mathbb{P}\left[ |M_3|> 8\sqrt{K\iota}+6\iota \right]\leq 2(\log_2(KH)+2)\delta.
\end{align}
\end{lemma}

\paragraph{Putting All Together}
By Lemma~\ref{lemma:add1},~\ref{lemma:bound_M1},~\ref{lemma:bound_M2} and \ref{lemma:bound_M3}, we conclude that, with probability $1-\big( 10S^2AH(\log_2(KH)+2)+ 6(\log_2(KH)+1)\log_2(KH) +1 \big) \delta$
\begin{align}
\text{Regret}(K) &\leq  M_1 +M_2 +M_3 +|\mathcal{K}^{C}| \nonumber
\\ & \leq O\left( \sqrt{SAKi_{\mathrm{max}}\iota}  + S^2A\iota\log_2(KH)   +\sqrt{K\iota}  +SA(\log_2(KH)+1)   \right)\nonumber
\\ & =O\left( \sqrt{SAK\log_2(KH)\iota}  +   S^2A\iota\log_2(KH)        \right). \nonumber
\end{align}
We finish the proof by rescaling $\delta$.


\section{Conclusion}
\label{sec:conclusion}
In this paper, we gave the first computationally efficient algorithm for tabular, episodic RL whose sample complexity scales logarithmically with $H$.
Furthermore, this algorithm matches the lower bound of a simpler problem, contextual bandits, up to logarithmic factors and an additive $S^2A$ term.
One important open problem is how to get rid of the additive $S^2A$ term (also see discussions in \cite{wang2020long}).
We remark that in the generative model setting, the optimal sample complexity does not have any additive term~\citep{agarwal2019optimality,li2020breaking}.

%


\bibliography{simonduref}

\begin{thebibliography}{39}
\providecommand{\natexlab}[1]{#1}
\providecommand{\url}[1]{\texttt{#1}}
\expandafter\ifx\csname urlstyle\endcsname\relax
  \providecommand{\doi}[1]{doi: #1}\else
  \providecommand{\doi}{doi: \begingroup \urlstyle{rm}\Url}\fi

\bibitem[Agarwal et~al.(2019)Agarwal, Kakade, and Yang]{agarwal2019optimality}
Alekh Agarwal, Sham Kakade, and Lin~F Yang.
\newblock On the optimality of sparse model-based planning for {Markov}
  decision processes.
\newblock \emph{arXiv preprint arXiv:1906.03804}, 2019.

\bibitem[Agrawal and Jia(2017)]{agrawal2017optimistic}
Shipra Agrawal and Randy Jia.
\newblock Optimistic posterior sampling for reinforcement learning: worst-case
  regret bounds.
\newblock In \emph{Advances in Neural Information Processing Systems}, pages
  1184--1194, 2017.

\bibitem[Azar et~al.(2017)Azar, Osband, and Munos]{azar2017minimax}
Mohammad~Gheshlaghi Azar, Ian Osband, and R{\'e}mi Munos.
\newblock Minimax regret bounds for reinforcement learning.
\newblock In \emph{Proceedings of the 34th International Conference on Machine
  Learning-Volume 70}, pages 263--272. JMLR. org, 2017.

\bibitem[Bartlett and Tewari(2009)]{bartlett2009regal}
Peter~L Bartlett and Ambuj Tewari.
\newblock Regal: a regularization based algorithm for reinforcement learning in
  weakly communicating mdps.
\newblock In \emph{Proceedings of the 25th Conference on Uncertainty in
  Artificial Intelligence (UAI 2009))}, 2009.

\bibitem[Brafman and Tennenholtz(2003)]{brafman2002r}
Ronen~I. Brafman and Moshe Tennenholtz.
\newblock R-max - a general polynomial time algorithm for near-optimal
  reinforcement learning.
\newblock \emph{J. Mach. Learn. Res.}, 3\penalty0 (Oct):\penalty0 213--231,
  March 2003.
\newblock ISSN 1532-4435.

\bibitem[Bubeck and Cesa-Bianchi(2012)]{bubeck2012regret}
Sebastien Bubeck and Nicolo Cesa-Bianchi.
\newblock Regret analysis of stochastic and nonstochastic multi-armed bandit
  problems.
\newblock \emph{Foundations and Trends in Machine Learning}, 5\penalty0
  (1):\penalty0 QT06,1--7,9--21,23--43,45--65,67--105,107--115,117--127, 2012.

\bibitem[Cai et~al.(2019)Cai, Yang, Jin, and Wang]{cai2019provably}
Qi~Cai, Zhuoran Yang, Chi Jin, and Zhaoran Wang.
\newblock Provably efficient exploration in policy optimization.
\newblock \emph{arXiv preprint arXiv:1912.05830}, 2019.

\bibitem[Dann and Brunskill(2015)]{dann2015sample}
Christoph Dann and Emma Brunskill.
\newblock Sample complexity of episodic fixed-horizon reinforcement learning.
\newblock In \emph{Advances in Neural Information Processing Systems}, pages
  2818--2826, 2015.

\bibitem[Dann et~al.(2017)Dann, Lattimore, and Brunskill]{dann2017unifying}
Christoph Dann, Tor Lattimore, and Emma Brunskill.
\newblock Unifying {PAC} and regret: Uniform {PAC} bounds for episodic
  reinforcement learning.
\newblock In \emph{Proceedings of the 31st International Conference on Neural
  Information Processing Systems}, NIPS’17, page 5717–5727, Red Hook, NY,
  USA, 2017. Curran Associates Inc.
\newblock ISBN 9781510860964.

\bibitem[Dann et~al.(2019)Dann, Li, Wei, and Brunskill]{dann2019policy}
Christoph Dann, Lihong Li, Wei Wei, and Emma Brunskill.
\newblock Policy certificates: Towards accountable reinforcement learning.
\newblock In \emph{Proceedings of the 36th International Conference on Machine
  Learning}, volume~97 of \emph{Proceedings of Machine Learning Research},
  pages 1507--1516, Long Beach, California, USA, 09--15 Jun 2019. PMLR.

\bibitem[Dong et~al.(2019)Dong, Wang, Chen, and Wang]{dong2019q}
Kefan Dong, Yuanhao Wang, Xiaoyu Chen, and Liwei Wang.
\newblock Q-learning with ucb exploration is sample efficient for
  infinite-horizon mdp.
\newblock \emph{arXiv preprint arXiv:1901.09311}, 2019.

\bibitem[Freedman(1975)]{freedman1975tail}
David~A Freedman.
\newblock On tail probabilities for martingales.
\newblock \emph{the Annals of Probability}, 3\penalty0 (1):\penalty0 100--118,
  1975.

\bibitem[Fruit et~al.(2018)Fruit, Pirotta, and Lazaric]{fruit2018near}
Ronan Fruit, Matteo Pirotta, and Alessandro Lazaric.
\newblock Near optimal exploration-exploitation in non-communicating markov
  decision processes.
\newblock In \emph{Advances in Neural Information Processing Systems}, pages
  2994--3004, 2018.

\bibitem[Jaksch et~al.(2010)Jaksch, Ortner, and Auer]{jaksch2010near}
Thomas Jaksch, Ronald Ortner, and Peter Auer.
\newblock Near-optimal regret bounds for reinforcement learning.
\newblock \emph{Journal of Machine Learning Research}, 11\penalty0
  (Apr):\penalty0 1563--1600, 2010.

\bibitem[Jiang and Agarwal(2018)]{jiang2018open}
Nan Jiang and Alekh Agarwal.
\newblock Open problem: The dependence of sample complexity lower bounds on
  planning horizon.
\newblock In \emph{Conference On Learning Theory}, pages 3395--3398, 2018.

\bibitem[Jin et~al.(2018)Jin, Allen-Zhu, Bubeck, and Jordan]{jin2018q}
Chi Jin, Zeyuan Allen-Zhu, Sebastien Bubeck, and Michael~I Jordan.
\newblock Is {Q}-learning provably efficient?
\newblock In \emph{Advances in Neural Information Processing Systems}, pages
  4863--4873, 2018.

\bibitem[Kakade(2003)]{kakade2003sample}
Sham~M Kakade.
\newblock \emph{On the sample complexity of reinforcement learning}.
\newblock PhD thesis, University of London London, England, 2003.

\bibitem[Kearns and Singh(2002)]{kearns2002near}
Michael Kearns and Satinder Singh.
\newblock Near-optimal reinforcement learning in polynomial time.
\newblock \emph{Machine learning}, 49\penalty0 (2-3):\penalty0 209--232, 2002.

\bibitem[Kolter and Ng(2009)]{kolter2009near}
J~Zico Kolter and Andrew~Y Ng.
\newblock Near-bayesian exploration in polynomial time.
\newblock In \emph{Proceedings of the 26th annual international conference on
  machine learning}, pages 513--520, 2009.

\bibitem[Lattimore and Hutter(2012)]{lattimore2012pac}
Tor Lattimore and Marcus Hutter.
\newblock Pac bounds for discounted mdps.
\newblock In \emph{International Conference on Algorithmic Learning Theory},
  pages 320--334. Springer, 2012.

\bibitem[Li et~al.(2020)Li, Wei, Chi, Gu, and Chen]{li2020breaking}
Gen Li, Yuting Wei, Yuejie Chi, Yuantao Gu, and Yuxin Chen.
\newblock Breaking the sample size barrier in model-based reinforcement
  learning with a generative model.
\newblock \emph{arXiv preprint arXiv:2005.12900}, 2020.

\bibitem[Maurer and Pontil(2009)]{maurer2009empirical}
Andreas Maurer and Massimiliano Pontil.
\newblock Empirical {B}ernstein bounds and sample variance penalization.
\newblock \emph{arXiv preprint arXiv:0907.3740}, 2009.

\bibitem[Neu and Pike-Burke(2020)]{neu2020unifying}
Gergely Neu and Ciara Pike-Burke.
\newblock A unifying view of optimism in episodic reinforcement learning.
\newblock \emph{arXiv preprint arXiv:2007.01891}, 2020.

\bibitem[Osband and Roy(2016)]{osband2016on}
Ian Osband and Benjamin~Van Roy.
\newblock On lower bounds for regret in reinforcement learning.
\newblock \emph{ArXiv}, abs/1608.02732, 2016.

\bibitem[Osband and Van~Roy(2017)]{osband2017posterior}
Ian Osband and Benjamin Van~Roy.
\newblock Why is posterior sampling better than optimism for reinforcement
  learning?
\newblock In \emph{Proceedings of the 34th International Conference on Machine
  Learning-Volume 70}, pages 2701--2710. JMLR. org, 2017.

\bibitem[Osband et~al.(2013)Osband, Russo, and Van~Roy]{osband2013more}
Ian Osband, Daniel Russo, and Benjamin Van~Roy.
\newblock (more) efficient reinforcement learning via posterior sampling.
\newblock In \emph{Advances in Neural Information Processing Systems}, pages
  3003--3011, 2013.

\bibitem[Pacchiano et~al.(2020)Pacchiano, Ball, Parker-Holder, Choromanski, and
  Roberts]{pacchiano2020optimism}
Aldo Pacchiano, Philip Ball, Jack Parker-Holder, Krzysztof Choromanski, and
  Stephen Roberts.
\newblock On optimism in model-based reinforcement learning.
\newblock \emph{arXiv preprint arXiv:2006.11911}, 2020.

\bibitem[Russo(2019)]{russo2019worst}
Daniel Russo.
\newblock Worst-case regret bounds for exploration via randomized value
  functions.
\newblock In \emph{Advances in Neural Information Processing Systems}, pages
  14433--14443, 2019.

\bibitem[Simchowitz and Jamieson(2019)]{simchowitz2019non}
Max Simchowitz and Kevin~G Jamieson.
\newblock Non-asymptotic gap-dependent regret bounds for tabular mdps.
\newblock In \emph{Advances in Neural Information Processing Systems}, pages
  1153--1162, 2019.

\bibitem[Strehl and Littman(2008)]{strehl2008analysis}
Alexander~L Strehl and Michael~L Littman.
\newblock An analysis of model-based interval estimation for markov decision
  processes.
\newblock \emph{Journal of Computer and System Sciences}, 74\penalty0
  (8):\penalty0 1309--1331, 2008.

\bibitem[Strehl et~al.(2006)Strehl, Li, Wiewiora, Langford, and
  Littman]{strehl2006pac}
Alexander~L Strehl, Lihong Li, Eric Wiewiora, John Langford, and Michael~L
  Littman.
\newblock Pac model-free reinforcement learning.
\newblock In \emph{Proceedings of the 23rd international conference on Machine
  learning}, pages 881--888. ACM, 2006.

\bibitem[Szita and Szepesv{\'a}ri(2010)]{szita2010model}
Istv{\'a}n Szita and Csaba Szepesv{\'a}ri.
\newblock Model-based reinforcement learning with nearly tight exploration
  complexity bounds.
\newblock In \emph{ICML}, 2010.

\bibitem[Talebi and Maillard(2018)]{talebi2018variance}
Mohammad~Sadegh Talebi and Odalric-Ambrym Maillard.
\newblock Variance-aware regret bounds for undiscounted reinforcement learning
  in mdps.
\newblock \emph{arXiv preprint arXiv:1803.01626}, 2018.

\bibitem[Wang et~al.(2020)Wang, Du, Yang, and Kakade]{wang2020long}
Ruosong Wang, Simon~S Du, Lin~F Yang, and Sham~M Kakade.
\newblock Is long horizon reinforcement learning more difficult than short
  horizon reinforcement learning?
\newblock \emph{arXiv preprint arXiv:2005.00527}, 2020.

\bibitem[Yang et~al.(2020)Yang, Yang, and Du]{yang2020q}
Kunhe Yang, Lin~F Yang, and Simon~S Du.
\newblock {$Q$}-learning with logarithmic regret.
\newblock \emph{arXiv preprint arXiv:2006.09118}, 2020.

\bibitem[Zanette and Brunskill(2019)]{zanette2019tighter}
Andrea Zanette and Emma Brunskill.
\newblock Tighter problem-dependent regret bounds in reinforcement learning
  without domain knowledge using value function bounds.
\newblock In \emph{International Conference on Machine Learning}, pages
  7304--7312, 2019.

\bibitem[Zhang and Ji(2019)]{zhang2019regret}
Zihan Zhang and Xiangyang Ji.
\newblock Regret minimization for reinforcement learning by evaluating the
  optimal bias function.
\newblock In \emph{Advances in Neural Information Processing Systems}, pages
  2823--2832, 2019.

\bibitem[Zhang et~al.(2020{\natexlab{a}})Zhang, Zhou, and Ji]{zhang2020almost}
Zihan Zhang, Yuan Zhou, and Xiangyang Ji.
\newblock Almost optimal model-free reinforcement learning via
  reference-advantage decomposition.
\newblock \emph{arXiv preprint arXiv:2004.10019}, 2020{\natexlab{a}}.

\bibitem[Zhang et~al.(2020{\natexlab{b}})Zhang, Zhou, and Ji]{zhang2020model}
Zihan Zhang, Yuan Zhou, and Xiangyang Ji.
\newblock Model-free reinforcement learning: from clipped pseudo-regret to
  sample complexity.
\newblock \emph{arXiv preprint arXiv:2006.03864}, 2020{\natexlab{b}}.

\end{thebibliography}
\bibliographystyle{plainnat}

\newpage 
\appendix
\renewcommand{\appendixname}{Appendix~\Alph{section}}
\setlength{\parindent}{0pt}
\setlength{\parskip}{0.2\baselineskip}
\section{Technical Lemmas}
\begin{lemma}[Bennet's Inequality]\label{bennet}
Let $Z,Z_1,...,Z_n$  be i.i.d. random variables with values in $[0,1]$ and let $\delta>0$. Define $\mathbb{V}Z = \mathbb{E}\left[(Z-\mathbb{E}Z)^2 \right]$. Then we have
\begin{align}
\mathbb{P}\left[ \left|\mathbb{E}\left[Z\right]-\frac{1}{n}\sum_{i=1}^n Z_i  \right| > \sqrt{\frac{  2\mathbb{V}Z \ln(2/\delta)}{n}} +\frac{\ln(2/\delta)}{n} \right]]\leq \delta.\nonumber
\end{align}
\end{lemma}

\begin{lemma}[Theorem 4 in  \cite{maurer2009empirical}  ]\label{empirical bernstein}
Let $Z,Z_1,...,Z_n$ ($n\geq 2$) be i.i.d. random variables with values in $[0,1]$ and let $\delta>0$. Define $\bar{Z} = \frac{1}{n}\sum_{i=1}^n Z_{i}$ and $\hat{V}_n  = \frac{1}{n}\sum_{i=1}^n (Z_i- \bar{Z})^2$. Then we have
\begin{align}
\mathbb{P}\left[ \left|\mathbb{E}\left[Z\right]-\frac{1}{n}\sum_{i=1}^n Z_i  \right| > \sqrt{\frac{  2\hat{V}_n \ln(2/\delta)}{n-1}} +\frac{7\ln(2/\delta)}{3(n-1)} \right] \leq \delta.\nonumber
\end{align}
\end{lemma}

\begin{lemma}[Lemma 10 in \cite{zhang2020model}]\label{self-norm}
Let $(M_n)_{n\geq 0}$ be a martingale such that $M_0=0$ and $|M_n-M_{n-1}|\leq c$ for some $c>0$ and any $n\geq 1$. Let $\mathrm{Var}_{n} = \sum_{k=1}^n \mathbb{E}\left[  (M_{k}-M_{k-1})^2 |\mathcal{F}_{k-1}\right]$ for $n\geq 0$, where $\mathcal{F}_k = \sigma(M_1,...,M_{k})$. Then for any positive integer $n$, and any $\epsilon,\delta>0$, we have that
\begin{align}
\mathbb{P} \left[       |M_n|\geq 2\sqrt{2}\sqrt{\mathrm{Var}_n \ln(1/\delta)} +2\sqrt{\epsilon \ln(1/\delta)} +2c\ln(1/\delta) \right]\leq 2(\log_2(\frac{nc^2}{\epsilon}) +1)\delta.\nonumber
\end{align}
\end{lemma}

\textbf{Lemma 2} [Restatement]
\emph{
Let $\lambda_1,\lambda_2,\lambda_4\geq 0$, $\lambda_3\geq 1$ and $i' =\log_2(\lambda_1)$.	Let $a_{1},a_{2},...,a_{i'}$ be non-negative reals such that $a_{i}\leq \lambda_{1}$ and $a_{i}\leq \lambda_{2}\sqrt{a_{i+1}+ 2^{i+1} \lambda_3 } +\lambda_4$ for any $1\leq i\leq i'$. 
	 Then we have that $a_{1}\leq \max\{ (\lambda_2 +\sqrt{\lambda_2^2+\lambda_4} )^2  ,\lambda_{2}\sqrt{8\lambda_3}  +\lambda_4 \}$   .}

\begin{proof}
	Let $i_0$ be the least integer such that $2^i \lambda_3 >\lambda_1$ and $i_1 =\max\{ i| i\leq i_0, a_{i}> 2^{i}\lambda_3 \} \cup\{0\}$.  Because $\lambda_3\geq 1$, $i_0\leq i'$.
	If $i_1\leq 1$, then we have $a_2\leq 4\lambda_3$. Otherwise, by definition, we have 
	\begin{align}
	2^{i_1}\lambda_3 < a_{i_1}\leq \lambda_2\sqrt{a_{i+1}+ 2^{i+1}\lambda_3}+\lambda_4  \leq  \lambda_2 2^{\frac{i_1+2}{2}}\sqrt{\lambda_3}+ \lambda_4,\nonumber
	\end{align}
	which implies that $(\sqrt{2^{i_1}\lambda_3})^2 < 2\lambda_2 \sqrt{2^{i_1}\lambda_3} +\lambda_4 ,$ and thus
	\begin{align}
	2^{i_1} \lambda_3 <a_{i_1}< \bar{a}:=(\lambda_2 +\sqrt{\lambda_2^2+\lambda_4} )^2 .\nonumber
	\end{align}
	For $1\leq i<i_1$, we have that
	\begin{align}
	a_{i}< \lambda_2 \sqrt{a_{i+1}+\bar{a}}+\lambda_4.\nonumber
	\end{align}
	Because $a_{i_1}< \bar{a}$, we have $a_{i_1-1}<\lambda_2 \sqrt{2\bar{a}}+\lambda_4\leq \bar{a}$. By induction, we have that $a_{2}<\bar{a}$ . Therefore, $a_2\leq \max\{ \bar{a}, 4\lambda_3 \}$ and  $a_{1}\leq \max\{\bar{a}, \lambda_{2}\sqrt{8\lambda_3}  +\lambda_4\}$.
	The proof is completed.
\end{proof}

\section{Missing Proofs in Section \ref{sec:main_proof_1}}

\subsection{Proof of Lemma~\ref{lemma:high_prob_e1_e2}}
\begin{proof}
Next, we will show that $\mathcal{E}_1$ and $\mathcal{E}_2$ hold with high probability. For each $(s,a)$, when $n^k(s,a)=1$ or $2$, \eqref{eq_lemma1_0} and \eqref{eq_lemma1_0.5} hold trivially. For $n^k(s,a)=2^i$ with $i\geq 2$, by Lemma \ref{empirical bernstein}, we have that
\begin{align}
	& \mathbb{P}\left[|(\hat{P}^k_{s,a}-P_{s,a})V_{h+1}^*| >   2\sqrt{\frac{ \mathbb{ V}(\hat{P}^k_{s,a}, V_{h+1}^*  )\iota }{ n^k(s,a)}} +\frac{14\iota}{3n^k(s,a)} \right] \nonumber
	\\ & \leq  \mathbb{P}\left[|(\hat{P}^k_{s,a}-P_{s,a})V_{h+1}^*| >   \sqrt{\frac{2 \mathbb{ V}(\hat{P}^k_{s,a}, V_{h+1}^*  )\iota }{ n^k(s,a)-1}} +\frac{7\iota}{3n^k(s,a)-1} \right] \nonumber
	\\ & \leq \delta \label{eq_lemma1_ref.5}
\end{align}
and
\begin{align}
	& \mathbb{P}\left[ |\hat{r}^k_h(s,a)-r(s,a)| > 2\sqrt{ \frac{2\hat{r}^k_h(s,a)\iota }{n^k(s,a)}} +\frac{28\iota}{3n^k(s,a)}  \right]\nonumber
	\\ & \leq \mathbb{P}\left[ |\hat{r}^k_h(s,a)-r(s,a)|  >2\sqrt{\frac{  \hat{\mathrm{Var}}_h^k(s,a) \iota}{n^k(s,a)-1}} +\frac{14\iota}{3(n^k(s,a)-1)}     \right]\nonumber
	\\ & \leq \delta,\label{eq_lemma1_ref1}
\end{align}
where $\hat{\mathrm{ Var}}_h^k(s,a) \leq \hat{r}_h^k(s,a)$\footnote{ $\mathbb{E}\left[(Z-\mathbb{E}[Z])^2 \right]\leq  \mathbb{E}\left[Z\right]$ for $Z\in [0,1]$.} is the empirical variance of $R(s,a)$ computed by the $n^k(s,a)$ samples.
Via a union bound over all $(s,a)$ and $i$, we obtain that $\mathbb{P}\left[\mathcal{E}_1\cap \mathcal{E}_2\right]\geq 1 -2SA(\log_{2}KH+1)\delta$. The proof is completed.
\end{proof}

\subsection{Proof of Lemma~\ref{lemma1}}

\begin{proof}
The proof of the two principles, optimism and monotonicity rely on exploiting the properties of the following $f$ defined in the following lemma.
\begin{lemma}\label{lemma:f_properties}
Let $f: \Delta^{S} \times \mathbb{R}^S \times \mathbb{R} \times \mathbb{R} \rightarrow \mathbb{R}$ with $f(p,v,n,\iota) =pv+ \max\left\{\bar{c}_1\sqrt{\frac{ \mathbb{ V}(p,v) \iota }{n }} ,\bar{c}_2\frac{\iota}{n} \right\}$ with $\bar{c}_1= \frac{20}{3}$ and $\bar{c}_2 = \frac{400}{9}$.
Then  $f$ satisfies
\begin{enumerate}
\item $f(p,v,n,\iota)$ is non-decreasing in $v(s)$  for all $p\in \Delta^{S}$,$\|v\|_{\infty}\leq 1$  and $n,\iota>0$;
\item $f(p,v,n,\iota)\geq pv +  2\sqrt{\frac{ \mathbb{ V}(p,v) \iota }{n }} +\frac{14\iota}{3n}$ for all $p,v$ and $n,\iota>0$. 
\end{enumerate}
\end{lemma}
The proof of the lemma is straightforward.
Note the first property is exactly the monotonicity we want.
\begin{proof}
To verify the first claim, we fix all other variables but $v(s)$ and view $f$ as a function in $v(s)$. Because the derivative of $f$ in $v(s)$ does not exist only when $ c_1\sqrt{  \frac{ \mathbb{V}(p,v)  \iota} {n}  }= c_2\frac{\iota}{n} $, where the condition has at most two solutions, so it suffices to prove $\frac{\partial f}{\partial v(s)}\geq 0$ when $ c_1\sqrt{  \frac{ \mathbb{V}(p,v)  \iota} {n}  }\neq  c_2\frac{\iota}{n} $. Direct computation gives that
\begin{align}
	\frac{\partial f}{\partial v(s)} &=p(s)+c_{1}\mathbb{I}\left[  c_1 \sqrt{  \frac{ \mathbb{V}(p,v)  \iota} {n}  }\geq  c_2\frac{\iota}{n} \right] \frac{ p(s)(v(s)- pv )\iota}{\sqrt{n\mathbb{V}(p,v) \iota}}\nonumber
	\\ & \geq \min\{ p(s)+ \frac{c_1^2}{c_2} p(s)(v(s)-pv),p(s) \}\nonumber
	\\ & \geq p(s) (1-\frac{c_1^2}{c_2})\nonumber
	\\ & = 0.
\end{align}

The second claim holds because both $\sqrt{\frac{ \mathbb{ V}(p,v) \iota }{n }}$ and $\frac{\iota}{n} $ are non-negative.
\end{proof}

Recall  we chose $c_1= \frac{460}{9}$, $c_2 = 2\sqrt{2}$ and $c_3 = \frac{544}{9}$.  Now we prove $Q_h^k(s,a)\geq Q^*_h(s,a)$ by backward induction conditioned on the event $\mathcal{E}_1$ and $\mathcal{E}_2$ hold. Firstly,  the conclusion holds for $h=H+1$ because $Q^*_{H+1}=0$. For $1\leq h\leq H$, assuming the conclusion holds for $h+1$, by  \eqref{equpdate2}, we have that
\begin{align}
& \quad Q_{h}^k(s,a)\\
&=\min\{ \hat{r}_h^k(s,a) +\hat{P}_{s,a}^k V_{h+1}^k +b_{h}^k(s,a) ,1\}\nonumber
\\ & \geq \min\{ \hat{r}_h^k(s,a) +\hat{P}_{s,a}^k V_{h+1}^k +b_{h}^k(s,a) ,Q_h^*(s,a)\}\nonumber\\
& \ge \min\{ \hat{r}_h^k(s,a) +\hat{P}_{s,a}^k V_{h+1}^k +c_1 \sqrt{\frac{   \mathbb{ V}(\hat{P}_{s,a} ,V^k_{h+1}) \iota  }{ n^k(s,a) }}+ c_{2}\sqrt{ \frac{ \hat{r}(s,a)\iota }{n^k(s,a) } }+c_3\frac{\iota}{ n^{k}(s,a) } ,Q_h^*(s,a)\} \label{eq_def_b}
\\ & \geq \min\{r(s,a) + \hat{P}_{s,a}^k V_{h+1}^k + \max\{ \bar{c}_1\sqrt{ \frac{ \mathbb{V}(\hat{P}_{s,a}^k, V_{h+1}^k)\iota }{ n^k(s,a) } } ,\bar{c}_{2}\frac{\iota}{n^k(s,a)}   \}       ,Q_h^*(s,a)   \} \label{eq_lemma1_75}
\\ & \geq \min \{    r(s,a)+\hat{P}_{s,a}^k V_{h+1}^* + \max\{ \bar{c}_1\sqrt{ \frac{ \mathbb{V}(\hat{P}_{s,a}^k, V_{h+1}^*)\iota }{ n^k(s,a) } } ,\bar{c}_{2}\frac{\iota}{n^k(s,a)}   \}        ,Q_h^*(s,a)\} \label{eq_lemma1_1}
\\ & \geq \min\{  r(s,a)+ \hat{P}_{s,a}^kV_{h+1}^* +  2\sqrt{\frac{ \mathbb{ V}(\hat{P}^k_{s,a}, V_{h+1}^*  )\iota }{ n^k(s,a)}} +\frac{14\iota}{3n^k(s,a)}  ,Q_h^*(s,a)  \} \label{eq_lemma1_2}
\\ & \geq \min\{    r(s,a)+ P_{s,a}V_{h+1}^*  , Q_h^*(s,a)\} \label{eq_lemma1_3}
\\ & = Q_h^*(s,a).\nonumber
\end{align}	
\eqref{eq_def_b} is by the definition of $b_h^k(s,a)$ and $n^{k}(s,a)$.
\eqref{eq_lemma1_75} is by the definition of $\mathcal{E}_2$ and our choice of $c_1,c_2,c_3$ and $\bar{c}_1,\bar{c}_2$.
\eqref{eq_lemma1_1} is by recognizing $f(\hat{P}_{s,a}^k, V_{h+1}^k, n^k(s,a),\iota) =\hat{P}_{s,a}^k V_{h+1}^k + \max\{ \bar{c}_1\sqrt{ \frac{ \mathbb{V}(\hat{P}_{s,a}^k, V_{h+1}^k)\iota }{ n^k(s,a) } } ,\bar{c}_{2}\frac{\iota}{n^k(s,a)} \}$, then using the first property in Lemma~\ref{lemma:f_properties}  and the induction that $V_{h+1}^k\geq V^*_{h+1}$, \eqref{eq_lemma1_2} is by the second property of Lemma~\ref{lemma:f_properties} and  the definition of $\mathcal{E}_1$.

\end{proof}

\section{Missing Proofs in Section~\ref{sec:main_proof_2}}
\subsection{Proof of Lemma~\ref{lemma_be}}
\begin{proof} It suffices to verify \eqref{eq_be_0} for the first term in \textbf{RHS}. 
Under $\mathcal{E}_1 \cap \mathcal{E}_2$, we have that with probability $1-SAH(\log_2(KH)+1)\delta$,  for all $(s,a,h,k) \in \states \times \actions \times [H] \times [K]$:
\begin{align}
& \quad Q_h^k(s,a)-r(s,a)-P_{s,a}V_{h+1}^k \\
&\leq \hat{r}_h^k(s,a)-r(s,a)+ b_h^k(s,a)+ (\hat{P}^k_{s,a}-P_{s,a})(V^k_{h+1}-V^*_{h+1}) + (\hat{P}^k_{s,a}-P_{s,a})V_{h+1}^*\nonumber
\\& \leq  2b_h^k(s,a)+ (\hat{P}^k_{s,a}-P_{s,a})(V^k_{h+1}-V^*_{h+1}) + (\hat{P}^k_{s,a}-P_{s,a})V_{h+1}^*. \label{eq_be_1}
\end{align}

Fix $s,a,h,k$. When $n^k(s,a) = 1$, \eqref{eq_be_0} holds trivially. For $n^k(s,a)=2^{i}$ with $i\geq 1$, by Bennet's inequality (see Lemma \ref{bennet}) we have that for each $s'$
\begin{align}
\mathbb{P}\left[ |\hat{P}^k_{s,a,s'}-P_{s,a,s'}| > \sqrt{\frac{2P_{s,a,s'} \iota }{n^k(s,a) }} +\frac{\iota}{3n^k(s,a)} \right]\leq \delta\nonumber.
\end{align}
So with probability $1-S\delta$, we have that
\begin{align}
(\hat{P}^k_{s,a}-P_{s,a})(V_{h+1}^k-V_{h+1}^*)& = \sum_{s'}( \hat{P}_{s,a,s'}^k -P_{s,a,s'})  ( V_{h+1}^k(s')-V_{h+1}^*(s') -P_{s,a}(V_{h+1}^k -V_{h+1}^* )   ) \label{eq_sum_1}
\\ & \leq \sum_{s'} \sqrt{\frac{2P_{s,a,s'} \iota }{n^k(s,a) }} |V_{h+1}^k(s')-V_{h+1}^*(s')-P_{s,a}(V_{h+1}^k -V_{h+1}^* )| +\frac{S\iota}{3n^k(s,a)}\nonumber
\\ & \leq \sqrt{\frac{2S \mathbb{V}( P_{s,a},V_{h+1}^k-V_{h+1}^* ) }{n^k(s,a) }} +\frac{S\iota}{3n^k(s,a)},\label{eq_be_2}
\end{align}
where \eqref{eq_sum_1} holds because $\sum_{s'}\hat{P}_{s,a,s'}^k = \sum_{s'} P_{s,a,s'} = 1$ and
 \eqref{eq_be_2} holds by Cauchy-Schwartz inequality. 
 On the other hand, by Bennet's inequality (see Lemma \ref{bennet}) again, we obtain that
\begin{align}
\mathbb{P}\left[    | (\hat{P }^k_{s,a}-P_{s,a})V^*_{h+1}   |>  \sqrt{   \frac{ 2\mathbb{V}(P_{s,a},V_{h+1}^*)\iota }{n^k(s,a)} } +\frac{\iota}{3n^k(s,a)}  \right]\leq \delta.\label{eq_be_3}
\end{align}
Combining \eqref{eq_be_1}, \eqref{eq_be_2} and \eqref{eq_be_3} and via a union bound over $k,h,s,a$, we conclude that \eqref{eq_be_0} holds with probability $1-3S^2AH(\log_2(KH)+1)\delta$, and with $c_4 =\sqrt{2}, c_{5}=\sqrt{2}$ and $c_{6}=\frac{2}{3}$.
\end{proof}

\section{Missing Proofs in Section~\ref{sec:main_proof_3}}
\subsection{Proof of Lemma~\ref{lemma:add1}}
\begin{proof}
Direct computation gives that
\begin{align}
&\textrm{Regret}(K) := \sum_{k=1}^K \left(V^*_1(s_1^k) -V_1^{\pi^k}(s_1^k) \right) \nonumber
 \\&\quad  \leq \sum_{k=1}^K \left(  V_1^k(s_1^k) -V_1^{\pi^k}(s_1^k) \right) \nonumber
 \\ & \quad  = \sum_{k=1}^K \left(  \check{V}_1^k(s_1^k) -V_1^{\pi^k}(s_1^k) \right) \nonumber
\\ &  \quad =\sum_{ k=1}^{ K  }(\check{V}_1^k(s_1^k) -\sum_{h=1}^H \check{r}_h^k) +\sum_{k=1}^K (\sum_{h=1}^H \check{r}_h^k- V^{\pi^k}_1(s_1^k)) \nonumber
\\ & \quad = \sum_{k=1}^K \sum_{h=1}^H (P_{s_h^k,a_h^k}-\textbf{1}_{s_{h+1}^k}) \check{V}_{h+1}^k + \sum_{k=1}^K \sum_{h=1}^H (\check{V}_{h}^k(s_h^k)-\check{r}_h^k -P_{s_h^k,a_h^k}\check{V}_{h+1}^k  ) + \sum_{k=1}^K (\sum_{h=1}^H \check{r}_h^k- V^{\pi^k}_1(s_1^k)) \nonumber
\\ & \quad  \leq  \sum_{k=1}^K\sum_{h=1}^H  (P_{s_h^k,a_h^k}-\textbf{1}_{s_{h+1}^k})\check{V}_{h+1}^k+ \sum_{k=1}^K\sum_{h=1}^H \check{\beta}_h^k(s_h^k,a_h^k) + \sum_{k=1}^K (\sum_{h=1}^H \check{r}_h^k- V^{\pi^k}_1(s_1^k)) +|\mathcal{K}^{C}|.\label{eq_sec2_1}
\end{align}
Here the first inequality is due to our optimistic estimation of $Q$-function, and \eqref{eq_sec2_1} holds by \eqref{eq:sec2_0.5} and \eqref{eq:sec2_0.7} .

\end{proof}

\subsection{Proof of Lemma~\ref{lemma:bound_M1}}
\begin{proof}
We note that $M_1$ could be viewed as a martingale because $\check{V}_{h+1}^k$ is measurable with respective to $\mathcal{F}_{h}^k$ where $\mathcal{F}_h^k=\sigma\left(\{s_{h'}^{k'},a_{h'}^{k'},r_{h'}^{k'},s_{h'+1}^{k'}\}_{1\leq k'<k,1\leq h'\leq H} \cup \{s_{h'}^k,a_{h'}^k,r_{h'}^k\}_{1\leq h'\leq h-1}\cup \{s_h^k,a_h^k\} \right)$, i.e., all past trajectories before $s_{h+1}^k$ is rolled out.  To avoid polynomial dependence on $H$, we use a variance-dependent concentration inequality to bound this term instead of Hoeffding inequality (see Lemma \ref{self-norm}).
By Lemma \ref{self-norm} with $\epsilon =1$, we have that
\begin{align}
\mathbb{P}\left[     |M_1| > 2\sqrt{2 \sum_{k=1}^K\sum_{h=1}^H \mathbb{V}( P_{s_h^k,a_h^k},\check{V}_{h+1}^k )\iota } +6\iota   \right]\leq 2(\log_{2}(KH)+1)\delta. \label{eq_sec3_2}
\end{align}
To bound $M_1$, it suffices to bound $M_4: =\sum_{k=1}^K\sum_{h=1}^H \mathbb{V}( P_{s_h^k,a_h^k},\check{V}_{h+1}^k ) $. We will deal with this term  later. 
\end{proof}

\subsection{Proof of Lemma~\ref{lemma:bound_M2}}
\begin{proof}
Recall that
\[ \beta_{h}^k(s,a) = O\left( \sqrt{\frac{ \mathbb{V}(\hat{P}^k_{s,a},V_{h+1}^k \iota ) }{n^k(s,a) }} +\sqrt{\frac{ \mathbb{V}(P_{s,a},V_{h+1}^*) }{n^k(s,a) }}  +\sqrt{ \frac{S\mathbb{V}(P_{s,a}, V_{h+1}^k-V_{h+1}^* )\iota }{n^k(s,a)} } +\sqrt{\frac{\hat{r}_h^k(s,a)\iota }{n^k(s,a)}}+\frac{S\iota}{n^k(s,a)} \right)  . \]
By Lemma \ref{empirical bernstein}, we have 
\begin{align}
  \mathbb{P}\left[\hat{P}_{s,a,s'}^k > \frac{3}{2}P_{s,a,s'} +\frac{4\iota}{3n^k(s,a)}    \right] \leq \mathbb{P}\left[ \hat{P}_{s,a,s'}^k-P_{s,a,s'} > \sqrt{\frac{2P_{s,a,s'}\iota}{n^k(s,a)}} +\frac{\iota}{3n^k(s,a)} \right]  \leq \delta,
\end{align}
which implies that, with probability $1-2S^2AH(\log_2(KH)+1)\delta$, it holds that for each $k,h$
\begin{align}
    \mathbb{V}(\hat{P}^k_{s,a},V_{h+1}^k) & = \sum_{s'} \hat{P}^k_{s,a,s'} \left(V_{h+1}^k(s')-\hat{P}^k_{s,a}V_{h+1}^k \right)^2 \nonumber
    \\ & \leq \sum_{s'}\hat{P}^k_{s,a,s'} \left(V_{h+1}^k(s')-P_{s,a}V_{h+1}^k \right)^2 \nonumber
    \\ & \leq \sum_{s'}  \left(\frac{3}{2}P_{s,a,s'}+ \frac{4\iota}{3n^k(s,a)} \right)\cdot  \left(V_{h+1}^k(s')-P_{s,a}V_{h+1}^k \right)^2 \nonumber
    \\ & \leq \frac{3}{2}\mathbb{V}(P_{s,a},V_{h+1}^k)+\frac{4S\iota}{3n^k(s,a)}.\nonumber
\end{align}
Note that $\mathbb{ V}(P,X+Y)\leq 2(\mathbb{ V}(P,X)+ \mathbb{ V}(P,Y) )$ for any $P,X,Y$, we then have 
\begin{align} \beta_h^k(s,a)\leq O\left( \sqrt{\frac{ \mathbb{V}(P_{s,a},V_{h+1}^k \iota ) }{n^k(s,a) }} +  \sqrt{ \frac{S\mathbb{V}(P_{s,a}, V_{h+1}^k-V_{h+1}^* )\iota }{n^k(s,a)} } +\sqrt{\frac{\hat{r}_h^k(s,a)\iota }{n^k(s,a)}}+\frac{S\iota}{n^k(s,a)} \right).\label{eq_sec623_1}\end{align}
 Note that under the doubling epoch update framework, despite those episodes in which an update is triggered, the number of visits of $(s,a)$ between the $i$-th update of $\hat{P}_{s,a}$ and the $i+1$-th update of $\hat{P}_{s,a}$ do not exceeds $2^{i-1}$. More precisely,
 recalling the definition of  $\mathcal{K}$,  for any $(s,a)$ and any $i\geq 3$, we have
 \begin{align}
      \sum_{k=1}^H\sum_{h=1}^H  \mathbb{I}\left[ (s_h^k,a_h^k)=(s,a), n^k(s,a)=2^{i-1}\right]\cdot\mathbb{I}\left[ (k,h)\notin \mathcal{X} \right]\leq 2^{i-1}.\label{eq_sec3_a01}
 \end{align}
  Recall $i_{\mathrm{max}} = \max\{ i| 2^{i-1}\leq KH\} = \left\lfloor \log_2(KH) \right\rfloor +1$. 
  To facilitate the analysis, we first derive a general deterministic result.
  Let $w =\{w_h^k\geq 0|1\leq h \leq H,1\leq k\leq K\}$ be a group of non-negative weights such that $w_h^k\leq 1$ for any $(k,h) \in [H] \times [K]$ and $w_h^k = 0$ for any $(k,h)\in\mathcal{X}$. 
  Later we will set $w_h^k$ to be the products of $\mathbb{I}\left[ (k,h)\notin \mathcal{X}\right]$ with $\hat{r}_h^k(s_h^k,a_h^k) $,  $\mathbb{ V} (P_{s_h^k,a_h^k},V^*_{h+1})$,  $\mathbb{ V} (P_{s_h^k,a_h^k},V^k_{h+1})$ and $ \mathbb{ V}(P_{s_h^k,a_h^k}, V^k_{h+1}-V^*_{h+1}  )$.
  
  We can calculate 
\begin{align}
&\sum_{k=1}^K \sum_{h=1}^H\sqrt{ \frac{ w_h^k  }{n^k(s_h^k,a_h^k)  }} \nonumber
\\ &  \leq  \sum_{k=1}^K\sum_{h=1}^H   \sum_{s,a}\sum_{i=3}^{i_{\mathrm{max}}}\mathbb{I}\left[ (s_h^k,a_h^k)=(s,a), n^k(s,a)=2^{i-1}\right]  \sqrt{\frac{w_h^k}{2^{i-1}}}   +   8SA(\log_2(KH)+4)\nonumber
\\ & = \sum_{s,a}\sum_{i=3}^{i_{\mathrm{max}}} \frac{1}{\sqrt{2^{i-1}}} \sum_{k=1}^K\sum_{h=1}^H  \mathbb{I}\left[ (s_h^k,a_h^k)=(s,a), n^k(s,a)=2^{i-1}\right] \sqrt{w_h^k }+   8SA(\log_2(KH)+4)\nonumber
\\ &  \leq \sum_{s,a}\sum_{i=3}^{i_{\mathrm{max}}} \sqrt{\frac{ \sum_{k=1}^K\sum_{h=1}^H  \mathbb{I}\left[ (s_h^k,a_h^k)=(s,a), n^k(s,a)=2^{i-1}\right]   }{2^{i-1}}}\cdot \nonumber
\\  & \quad \quad \quad  \sqrt{  \left( \sum_{k=1}^K\sum_{h=1}^H  \mathbb{I}\left[ (s_h^k,a_h^k)=(s,a), n^k(s,a)=2^{i-1}\right]w_h^k \right)   }+  8SA(\log_2(KH)+4)\label{eq_sec3_01}
\\ & \leq \sqrt{ SAi_{\mathrm{max}}\sum_{k=1}^K \sum_{h=1}^H w_h^k   }+  8SA(\log_2(KH)+4).\label{eq_sec3_03} 
\end{align}
Here  \eqref{eq_sec3_01} is by Cauchy-Schwarz inequality and \eqref{eq_sec3_03} is by \eqref{eq_sec3_a01} and Cauchy-Schwarz inequality.

Let $I(k,h)$ be shorthand of $\mathbb{I}\left[ (k,h)\notin \mathcal{X}\right]$.  It is worth noting that by definition, $\sum_{k=1}^K \sum_{h=1}^H |I(k,h)-I(k,h+1)|\leq |\mathcal{K}^C|$.
By plugging respectively $w_h^k =I(k,h) \hat{r}_h^k(s_h^k,a_h^k) $,  $I(k,h) \mathbb{ V} (P_{s_h^k,a_h^k},V^*_{h+1})$,  $ I(k,h) \mathbb{ V} (P_{s_h^k,a_h^k},V^k_{h+1})$ and $ I(k,h) \mathbb{ V}(P_{s_h^k,a_h^k}, V^k_{h+1}-V^*_{h+1}  )$ into \eqref{eq_sec3_03},
and recalling \eqref{eq_sec623_1}, we obtain that 
\begin{align}
&M_2 = \sum_{k=1}^K\sum_{h=1}^H \check{\beta}_{h}^k(s_h^k,a_h^k)\nonumber
\\ & = \sum_{k=1}^K\sum_{h=1}^H \beta_{h}^k(s_h^k,a_h^k)I(k,h)
\\ & \leq  O\left(  \sqrt{  SAi_{\mathrm{max}} \iota\sum_{k=1}^K\sum_{h=1}^H  \mathbb{V}(P_{s_h^k,a_h^k},V_{h+1}^k) I(k,h)           }   + \sqrt{S^2Ai_{\mathrm{max}} \iota \sum_{k=1}^K \sum_{h=1}^H\mathbb{V}(P_{s_h^k,a_h^k},V_{h+1}^k-V^*_{h+1}) I(k,h)  }  \right) \nonumber
\\ & \quad + O\left(  \sqrt{SAi_{\mathrm{max}}\sum_{k=1}^K \sum_{h=1}^H \hat{r}_h^k(s_h^k,a_h^k)I(k,h)\iota} + S^2A\iota\log_2(KH)  \right) \label{eq_sec3_3}
\\ & \leq  O\left(  \sqrt{  SAi_{\mathrm{max}} \iota\sum_{k=1}^K\sum_{h=1}^H  \mathbb{V}(P_{s_h^k,a_h^k},V_{h+1}^k) I(k,h)           }   + \sqrt{S^2Ai_{\mathrm{max}} \iota \sum_{k=1}^K \sum_{h=1}^H\mathbb{V}(P_{s_h^k,a_h^k},V_{h+1}^k-V^*_{h+1})I(k,h)   }  \right) \nonumber
\\ & \quad + O\left(  \sqrt{SAi_{\mathrm{max}}K\iota } + S^2A\iota\log_2(KH)  \right) \label{eq_sec3_3.5}
\end{align}
where in \eqref{eq_sec3_3.5}, we used the following lemma whose proof is deferred to appendix.
\begin{lemma}\label{lemma_bdb}
	$\sum_{k=1 }^K \sum_{h=1}^H \hat{r}_h^k(s_h^k,a_h^k)I(k,h)\leq 2\sum_{k=1}^K \sum_{h=1}^H r_h^k + 4SA\leq 2K+4SA.$
\end{lemma}
\begin{proof}
	For any $(k,h)$ and $(k',h')$, we define 
	\[\tilde{w}_h^k(k',h') = \frac{1}{  n^{k}(s_h^k,a_h^k)}\mathbb{I}\left[ (s_h^k,a_h^k) =(s_{h'}^{k'} ,a_{h'}^{k'}) \right]\cdot\mathbb{I} \left[n^{k'}(s_{h'}^{k'}, a_{h'}^{k'} ) =2 n^{k}(s_h^k,a_h^k) \right]\cdot I(k',h').\]
	
	
By the update rule, for each $(k',h')$ pair with $n^{k'}( s_{h'}^{k'},a_{h'}^{k'}  )\geq 2$, $\hat{r}_{h'}^{k'}(s_{h'}^{k'}, a_{h'}^{k'}) = \sum_{k=1}^K \sum_{h=1}^H \tilde{w}_h^k(h',k') r_h^k$.  On the other hand, because $\tilde{w}_h^k(h',k')\leq \frac{1}{n^{k}(s_h^k,a_h^k)}$ for any $(k',h')$, and $\sum_{k'=1}^K\sum_{h'=1}^H \mathbb{I}\left[  \tilde{w}_h^k(h',k') > 0  \right]\leq 2n^{k}(s_h^k,a_h^k) $, we have 
\begin{align}
\sum_{k'=1}^K\sum_{h'=1}^H \tilde{w}_h^k(h',k') \leq 2.\nonumber
\end{align}
	Therefore, we have
	\begin{align}
	\sum_{k=1 }^K \sum_{h=1}^H \hat{r}_h^k(s_h^k,a_h^k) &\leq     	\sum_{k\in \mathcal{K} } \sum_{h=1}^H \mathbb{I}\left[n^k(s_h^k,a_h^k)\geq 2 \right] \hat{r}_h^k(s_h^k,a_h^k)  +  4SA \nonumber
	\\ & \leq 2 \sum_{k=1}^K \sum_{h=1}^H r_h^k +4SA \nonumber
	\\ & \leq 2K + 4SA.\nonumber
	\end{align}
	
We remark that  if we use the standard maximum likelihood estimation, the weight of the a reward would be $1\cdot 1+ 2\cdot \frac{1}{2}+4\cdot \frac{1}{4}+.... \approx \log(T)$. However, if we update the empirical reward using the latest half fraction of samples, the weight for each reward is only $2^{i+1} \frac{1}{2^i}\leq 2$. Therefore, we can save a $\log(T)$ factor. 	
	
\end{proof}

Recalling the definition of $M_4$, by the fact $\sum_{k=1}^K\sum_{h=1}^{H}|I(k,h+1)-I(k,h)|\leq |\mathcal{K}^{C}|$, we have that 
\begin{align}
M_4 &=\sum_{k=1}^K\sum_{h=1}^H \mathbb{V}( P_{s_h^k,a_h^k},\check{V}_{h+1}^k ) \nonumber
\\&=\sum_{k=1}^K\sum_{h=1}^H \mathbb{V}( P_{s_h^k,a_h^k},V_{h+1}^k )I(k,h+1)\nonumber 
\\& \geq \sum_{k=1}^K\sum_{h=1}^H \mathbb{V}( P_{s_h^k,a_h^k},V_{h+1}^k )I(k,h) -|\mathcal{K}^{C}|.\label{eq_sec_firef}
\end{align}
We further define $M_5 =  \sum_{k=1}^K \sum_{h=1}^H\mathbb{V}(P_{s_h^k,a_h^k},V_{h+1}^k-V^*_{h+1})I(k,h+1) $. Following similar arguments, we have that
\begin{align}
\sum_{k=1}^K \sum_{h=1}^H \mathbb{V}(P_{s_h^k,a_h^k},V_{h+1}^k-V_{h+1}^* )I(k,h)\leq M_5+ |\mathcal{K}^C|.\label{eq_sec_firef2}    
\end{align}
Bounding these two terms is one of the main difficulties in this paper, for which we need to use the recursion-based technique introduced in Section~\ref{sec:tac}. The following two lemmas bound these two terms.
 
\begin{lemma}\label{lemma:bdM4}
With probability $1-2(\log_2(KH)+1)\log_2(KH)\delta$, it holds that
\begin{align}
M_4 \leq 2M_2+2|\mathcal{K}^{C}|+2K+ \max\{    46\iota, 8\sqrt{ (M_2+|\mathcal{K}^{C}|+K)\iota } +6\iota          \}.\label{eq_sec3_9}
\end{align}
\end{lemma}

\begin{proof}
Direct computation gives that
\begin{align}
	&M_4 = \sum_{k=1}^K\sum_{h=1}^H \mathbb{V}(P_{s_h^k,a_h^k},V_{h+1}^k)I(k,h+1) \nonumber
	\\ & =\sum_{k=1}^K \sum_{h=1}^H \left( P_{s_h^k,a_h^k} (V_{h+1}^k)^2 -   (P_{s_h^k,a_h^k}V_{h+1}^k )^2  \right)I(k,h+1) \nonumber
	\\ &  = \sum_{k=1}^K\sum_{h=1}^H ( P_{s_h^k,a_h^k} (V_{h+1}^k)^2  -(V_{h+1}^k(s_{h+1}^k))^2  )I(k,h+1)  \nonumber
\\& \quad \quad 	+ \sum_{k=1}^K\sum_{h=1}^{H} \left(   (V_h^k(s_h^k) )^2   - (P_{s_h^k,a_h^k}V_{h+1}^k )^2       \right)I(k,h+1) - (V^k_1(s_1^k) )^2\nonumber
	\\ & \leq\sum_{k=1}^K \sum_{h=1}^H ( P_{s_h^k,a_h^k} (V_{h+1}^k)^2  -(V_{h+1}^k(s_{h+1}^k))^2  )I(k,h+1) +  2\sum_{k=1}^K\sum_{h=1}^H   \max\{ V_h^k(s_h^k) -P_{s_h^k,a_h^k}V_{h+1}^k ,0  \}I(k,h+1)\nonumber
	\\ & \leq \sum_{k=1}^K\sum_{h=1}^H ( P_{s_h^k,a_h^k} (V_{h+1}^k)^2  -(V_{h+1}^k(s_{h+1}^k))^2  )I(k,h+1) +  2\sum_{k=1}^K\sum_{h=1}^H( r(s_h^k,a_h^k) +\beta_h^k(s_h^k,a_h^k) )I(k,h+1)\label{eq_sec3_4}
	\\ & \leq   \sum_{k=1}^K\sum_{h=1}^H ( P_{s_h^k,a_h^k} (V_{h+1}^k)^2 I(k,h+1)  -(V_{h+1}^k(s_{h+1}^k))^2  ) + 2\sum_{k=1}^K\sum_{h=1}^H \beta_h^k(s_h^k,a_h^k)I(k,h)  + 2|\mathcal{K}^{C}| +2K \nonumber
	\\ & =  \sum_{k=1}^K\sum_{h=1}^H ( P_{s_h^k,a_h^k} (V_{h+1}^k)^2  -(V_{h+1}^k(s_{h+1}^k))^2  )I(k,h+1)+  2M_2 + 2|\mathcal{K}^{C}| +2K.\label{eq_sec3_5}
\end{align}
Here \eqref{eq_sec3_4} is by \eqref{eq_be_0} and \eqref{eq_sec3_5} is by the fact $\sum_{h=1}^H r(s_h^k,a_h^k)\leq 1$.

Define $F(m)  = \sum_{k=1}^K \sum_{h=1}^H  (    P_{s_h^k,a_h^k} ( V_{h+1}^k)^{2^m} - (V_{h+1}^k(s_{h+1}^k))^{2^m}     ) I(k,h+1) =  \sum_{k=1}^K \sum_{h=1}^H  (    P_{s_h^k,a_h^k} ( \check{V}_{h+1}^k)^{2^m} - (\check{V}_{h+1}^k(s_{h+1}^k))^{2^m}     ) $ for $1\leq m \leq \log_{2}(H)$. Because $\check{V}_{h+1}^k$ is measurable in $\mathcal{F}_h^k$, $F(m)$ can be viewed as a  martingale.
For a fixed $m$, by Lemma \ref{self-norm} with $\epsilon = 1$, we have that for each $m\leq \log_{2}(H)$, 
\begin{align}
	\mathbb{P}\left[ |F(m)|> 2\sqrt{2  \sum_{k=1}^K \sum_{h=1}^H   \mathbb{V} ( P_{s_h^k,a_h^k},     (\check{V}_{h+1}^k)^{2^{m}}     ) \iota      }+6\iota   \right]\leq 2(\log_2(KH)+1)\delta.\label{eq_sec3_6}
\end{align}
Note that
\begin{align}
	&\sum_{k=1}^K \sum_{h=1}^H   \mathbb{V} ( P_{s_h^k,a_h^k},     (\check{V}_{h+1}^k)^{2^{m}}     )  = \sum_{k=1}^K \sum_{h=1}^H  \left( P_{s_h^k,a_h^k} (V_{h+1}^k)^{2^{m+1}}  -  (P_{s_h^k,a_h^k}  (V_{h+1}^k)^{2^m}    )^2   \right)I(k,h+1)
	\nonumber\\ &\quad  = \sum_{k=1}^K \sum_{h=1}^H  (P_{s_h^k,a_h^k} -\textbf{1}_{s_{h+1}^k }) (V_{h+1}^k)^{2^{m+1}}I(k,h+1)\nonumber
\\ & \quad \quad \quad \quad 	+ \sum_{k=1}^K \sum_{h=1}^H \left( (V_{h}^k(s_h^k))^{2^{m+1}} -(P_{s_h^k,a_h^k}  (V_{h+1}^k)^{2^m})^2 I(k,h+1)    \right) - \sum_{k=1}^K (V_{1}^{k}(s_1^k))^{2^{m+1}} \nonumber
	\\ & \quad \leq  F(m+1) + \sum_{k=1}^K \sum_{h=1}^H \left(  (  V_{h}^k(s_h^k))^{2^{m+1}}  -(P_{s_h^k,a_h^k}   V_{h+1}^k )^{2^{m+1}} \right)I(k,h+1) \label{eq_sec3_7}
	\\ &\quad  \leq F(m+1)+2^{m+1}\sum_{k=1}^K\sum_{h=1}^H\max\{  V_h^k(s_h^k)-P_{s_h^k,a_h^k}V_{h+1}^k,0 \}I(k,h+1)\label{eq_sec3_7.5}
	\\ &\quad  \leq  F(m+1)+2^{m+1}\sum_{k=1}^K\sum_{h=1}^H\left( r(s_h^k,a_h^k)+\beta_h^k(s_h^k,a_h^k)\right)I(k,h+1) \nonumber
	\\& \quad \leq F(m+1)+ 2^{m+1}(\sum_{k=1}^K\sum_{h=1}^H \beta_h^k(s_h^k,a_h^k)I(k,h)+|\mathcal{K}^{C}|+ K)
	\\ &\quad  = F(m+1)+ 2^{m+1}(M_2 +|\mathcal{K}^{C}| +K)\label{eq_sec3_8}
\end{align}
Here \eqref{eq_sec3_7} is by convexity of $x^{2^{m}}$ and \eqref{eq_sec3_7.5} is by the fact $a^x-b^x\leq x\max\{a-b,0\}$ for $a,b\in [0,1]$.

Via a union bound over $m=1,2,...,\log_{2}(KH)$, we have that with probability $1- 2(\log_2(KH)+1 )\log_2(KH)\delta$,
\begin{align}
	F(m)\leq 2\sqrt{2(F(m+1) +2^{m+1}(M_2+ |\mathcal{K}^{C}|+K) )\iota  } +6\iota \label{eq:recurssion}
\end{align}
holds for any $1\leq m\leq \log_2(KH)$.
Now we have obtained a recursive formula.
In Lemma~\ref{lemma2}, we obtain the bound for the class of  recursive formulas of the same form as~\eqref{eq:recurssion}.
The proof of Lemma~\ref{lemma2} is deferred to appendix.
By \eqref{eq_sec3_5} and Lemma \ref{lemma2} with parameters $\lambda_{1}=KH$, $\lambda_2 = \sqrt{8\iota}$, $\lambda_3 = M_2+ |\mathcal{K}^{C}|+K$ and $\lambda_4 = 6\iota$, we have that with probability $1- 2(\log_2(KH)+1 )\log_2(KH)\delta$, 
\begin{align}
	M_4 \leq 2M_2+2|\mathcal{K}^{C}|+2K+ F(1)\leq 2M_2+2|\mathcal{K}^{C}|+2K +\max\{    46\iota, 8\sqrt{ (M_2+|\mathcal{K}^{C}|+K)\iota } +6\iota          \}.
\end{align}
\end{proof}

\begin{lemma}\label{lemma:bdM5}
	With probability $1-2(\log_2(KH)+1)\log_2(KH)\delta$, it holds that
	\begin{align}
		M_5\leq 2\max\{M_2,1\}  + 2|\mathcal{K}^{C}|+\max\{  46\iota, 8\sqrt{ (M_2+|\mathcal{K}^{C}|)\iota }+6\iota       \} .\label{eq_sec3_10}
	\end{align}
\end{lemma}

\begin{proof}
Recall that $\tilde{V}_{h+1}^k = V_{h+1}^k -V_{h+1}^*$.
We  compute
\begin{align}
&M_5 = \sum_{k=1}^K \sum_{h=1}^H \mathbb{ V}(P_{s,a} ,\tilde{V}_{h+1}^k )I(k,h+1)\nonumber
\\ & = \sum_{ k=1 }^{ K  } \sum_{h=1}^H \left( P_{s_h^k,a_h^k} (\tilde{V}_{h+1}^k)^2 -   (P_{s_h^k,a_h^k}\tilde{V}_{h+1}^k )^2  \right)I(k,h+1)\nonumber
\\ &  =  \sum_{ k=1 }^{ K  } \sum_{h=1}^H ( P_{s_h^k,a_h^k} (\tilde{V}_{h+1}^k)^2  -(\tilde{V}_{h+1}^k(s_{h+1}^k))^2  )I(k,h+1) \nonumber
\\& \quad \quad +\sum_{ k=1}^{ K  } \sum_{h=1}^{H} \left(   (\tilde{V}_h^k(s_h^k) )^2   -  (P_{s_h^k,a_h^k}\tilde{V}_{h+1}^k )^2       \right)I(k,h+1) - \sum_{ k=1}^{ K  }(\tilde{V}^k_1(s_1^k) )^2\nonumber
\\ & \leq  \sum_{ k=1 }^{ K  } \sum_{h=1}^H ( P_{s_h^k,a_h^k} (\tilde{V}_{h+1}^k)^2  -(\tilde{V}_{h+1}^k(s_{h+1}^k))^2  )I(k,h+1) +  2\sum_{ k=1}^{ K  }\sum_{h=1}^H \max\{ \tilde{V}_h^k(s_h^k) -P_{s_h^k,a_h^k}\tilde{V}_{h+1}^k,0 \}I(k,h+1) \label{eq_secpfbdM5_1}
\\ & \leq  \sum_{ k=1 }^{ K  } \sum_{h=1}^H \left( P_{s_h^k,a_h^k} (\tilde{V}_{h+1}^k)^2  -(\tilde{V}_{h+1}^k(s_{h+1}^k))^2  \right)I(k,h+1) +  2\sum_{ k=1}^{ K  }\sum_{h=1}^H \beta_h^k(s_h^k,a_h^k)I(k,h+1) \nonumber
\\ & \leq  \sum_{ k=1 }^{ K  } \sum_{h=1}^H \left( P_{s_h^k,a_h^k} (\tilde{V}_{h+1}^k)^2  -(\tilde{V}_{h+1}^k(s_{h+1}^k))^2  \right)I(k,h+1) +  2\sum_{ k=1}^{ K  }\sum_{h=1}^H \beta_h^k(s_h^k,a_h^k)I(k,h) +2|\mathcal{K}^{C}| \nonumber
\\ & \leq  \sum_{ k=1 }^{ K  } \sum_{h=1}^H \left( P_{s_h^k,a_h^k} (\tilde{V}_{h+1}^k)^2  -(\tilde{V}_{h+1}^k(s_{h+1}^k))^2 \right)I(k,h+1) +  2\max\{M_2,1\} +2|\mathcal{K}^{C}|.\nonumber
\end{align}
Here \eqref{eq_secpfbdM5_1} is by \eqref{eq_sec3_1}
Define $\tilde{F}(m) = \sum_{k=1}^K \sum_{h=1}^H  (    P_{s_h^k,a_h^k} ( \tilde{V}_{h+1}^k)^{2^m} - (\tilde{V}_{h+1}^k(s_{h+1}^k))^{2^m}     ) I(k,h+1)  $
. Following the same arguments in \eqref{eq_sec3_6} and \eqref{eq_sec3_8}, we obtain 
that with probability $1-2( \log_2(KH)+1 )\log_2(KH)\delta$, 
\begin{align}
    \tilde{F}(m)\leq 2\sqrt{2(\tilde{F}(m+1) +2^{m+1}(\max\{M_2,1\} +|\mathcal{K}^{C}|) )\iota  } +6\iota
\end{align}
holds for any $1\leq m\leq \log_2(KH)$.
By applying Lemma \ref{lemma2} with $\lambda_{1}=KH$, $\lambda_2 = \sqrt{8\iota}$, $\lambda_3 = (\max\{M_2,1\}+ |\mathcal{K}^{C}|)$ and $\lambda_4 = 6\iota$, we have that with probability $1-2( \log_2(KH)+1 )\log_2(KH)\delta$,
\begin{align}
M_5\leq2\max\{M_2,1\}+2|\mathcal{K}^{C}| + \tilde{F}(1)\leq 2\max\{M_2,1\}+ \max\{  46\iota, 8\sqrt{( M_2+|\mathcal{K}^{C}|)\iota }+6\iota       \} .
\end{align}
\end{proof}

Combining \eqref{eq_sec_firef} , \eqref{eq_sec_firef2}, \eqref{eq_sec3_3}, \eqref{eq_sec3_9} and \eqref{eq_sec3_10},  we have that with probability $1-\big( 6S^2AH(\log_2(KH)+1)+ 6(\log_2(KH)+1)\log_2(H) \big) \delta$,
\begin{align}
& M_2\leq O\left( \sqrt{SAi_{\mathrm{max}}(M_4+ |\mathcal{K}^{C}|)\iota   } +\sqrt{S^2Ai_{\mathrm{max}}(M_5+ |\mathcal{K}^{C}|)\iota  } +\sqrt{SAi_{\mathrm{max}}K\iota } +S^2A\iota\log_{2}(KH)\right),\label{eq_3_11}
\\ & M_4\leq 2M_2+2|\mathcal{K}^{C}|+2K+ \max\{46\iota, 8\sqrt{(M_2+2K)\iota}+6\iota \},\label{eq_3_12}
    \\ & M_5 \leq 2\max\{M_2,1\}+2|\mathcal{K}^{C}|+  \max\{46\iota,  \sqrt{ M_2\iota  } +6\iota \} .\label{eq_3_13}
\end{align}
These imply that 
\begin{align}
M_2& \leq O\left(  \sqrt{SAKi_{\mathrm{max}}\iota} + \sqrt{S^2Ai_{\mathrm{max}}\sqrt{M_2}\iota^{3/2}  } +\sqrt{SAi_{\mathrm{max}}K\iota }+ S^2A\iota\log_2(KH)       \right)\nonumber
\\& \leq  O\left( \sqrt{SAKi_{\mathrm{max}}\iota}  +   S^2A\iota\log_2(KH)        \right).\label{eq_3_14}
\end{align}

\end{proof}

\subsection{Proof of Lemma~\ref{lemma:bound_M3}} 
\begin{proof}
For the term $M_3$, we have
\begin{align}
M_3&= \sum_{k=1}^K \left(\sum_{h=1}^H \check{r}_h^k- V^{\pi^k}_1(s_1^k) \right) \nonumber
\\& = \sum_{k=1}^K \sum_{h=1}^H (\check{r}_h^k -r_h^k ) +\sum_{k=1}^K \left( \sum_{h=1}^H r_h^k - V^{\pi^k}_1(s_1^k)  \right) \nonumber
\\& \leq  \sum_{k=1}^K \sum_{h=1}^H (r(s_h^k,a_h^k) -r_h^k ) +\sum_{k=1}^K \left( \sum_{h=1}^H r_h^k - V^{\pi^k}_1(s_1^k)  \right).\label{eq_sec3_a1}
\end{align}

 For the first term in \textbf{RHS} of \eqref{eq_sec3_a1}, by Lemma \ref{self-norm}, we have that
\begin{align}
\mathbb{P}\left[    |\sum_{k=1}^K \sum_{h=1}^H (r(s_h^k,a_h^k) -r_h^k ) |> 2\sqrt{2\sum_{k=1}^K \sum_{h=1}^H \mathrm{Var}(s,a)\iota }   +6\iota        \right]\leq 2(\log_2(KH)+1)\delta ,
\label{eq_sec3_a3}
\end{align}
where $\mathrm{Var}(s,a): =\mathbb{E}\left[ (R(s,a)- \mathbb{E}[R(s,a)])^2 \right] $.
Since for a random variable $Z \in [0,1]$, $\mathrm{Var}\left[Z\right] \le \expect[Z]$, we have
\begin{align}
\sum_{k=1}^{K}\sum_{h=1}^H \mathrm{Var}(s,a)\leq \sum_{k=1}^K \sum_{h=1}^H r(s,a)\leq   \sum_{k=1}^K\sum_{h=1}^H (r(s_h^k,a_h^k)-r_h^k) + K ,\nonumber
\end{align}
Define $\bar{M}_3: =\sum_{k=1}^K \sum_{h=1}^H (r(s_h^k,a_h^k) -r_h^k )$.
We then have
\begin{align}
\mathbb{P}\left[    |\bar{M}_3 |> 2\sqrt{2 (\bar{M}_3+K)\iota }   +6\iota        \right]\leq 2(\log_2(KH)+1)\delta , \label{eq_sec3_a15}
\end{align}
which implies that  $|\bar{M}_3|\leq 6\sqrt{K\iota}+21\iota$ with probability  at least $1-2(\log_2(KH)+1)\delta$.

As for the second term in \textbf{RHS} of \eqref{eq_sec3_a1}, we define $Y_{k} =  \sum_{h=1}^H r_h^k - V^{\pi^k}_1(s_1^k)  $. Because for each $k$, $|Y_k|\leq 1$ and $\mathbb{E}\left[Y_k| \mathcal{F}^{k-1} \right]=0$, by Azuma's inequality, we have 
\begin{align}
\mathbb{P}\left[   \abs{\sum_{k=1}^K  Y_k}>\sqrt{2K\iota}    \right]\leq \delta.\label{eq_sec3_a2}
\end{align}
Combining \eqref{eq_sec3_a15} with \eqref{eq_sec3_a2}, we have that
\begin{align}
   \mathbb{P}\left[ |M_3|> 8\sqrt{K\iota}+6\iota \right]\leq 2(\log_2(KH)+2)\delta. \label{eq_sec3_a3.5}
\end{align}

\end{proof}

\end{document}